\documentclass[12pt]{article}

\usepackage{url}
\usepackage{hyperref}

\usepackage{comment}

\usepackage{appendix}
\usepackage{cancel} 
\usepackage{subcaption}
\usepackage[normalem]{ulem}
\usepackage{amsmath}
\usepackage{amsthm}
\usepackage{graphicx,psfrag,epsf, hyperref}
\usepackage{enumerate}
\usepackage{natbib}
\usepackage{booktabs}
\usepackage{rotating}
\usepackage[table]{xcolor}

\usepackage{fontenc}
\usepackage[T1]{fontenc}
\usepackage{xcolor}
\usepackage{textcomp}
\usepackage{algorithm}
\usepackage{algpseudocode}
\usepackage{amsmath, amssymb, bm}
\usepackage{relsize}
\usepackage{enumitem}
	
\usepackage{amsthm}
\theoremstyle{definition}
\newtheorem{theorem}{Theorem}[section]

\newtheorem{lemma}[theorem]{Lemma}

\theoremstyle{remark}
\newtheorem{remark}[theorem]{Remark}
\newtheorem{assumption}{Assumption}
\newcommand{\blind}{0}

\begin{document}

\def\spacingset#1{\renewcommand{\baselinestretch}%
{#1}\small\normalsize} \spacingset{1}


\if0\blind
{
  \title{\bf Automatic knot selection in smooth additive models}
  \author{Nicolás Carrizosa\thanks{Corresponding author: \href{mailto:carrizosanicolas@uniovi.es}{carrizosanicolas@uniovi.es}}\\
    { \small Department of Statistics, Operations Research and Mathematics Education, Universidad de Oviedo} \\
    \vspace*{0.1cm}
    Vanesa Guerrero\\
    {\small Department of Statistics, Universidad Carlos III de Madrid} \\
    \vspace*{0.1cm}
    María Durbán\\
    {\small Department of Statistics, Universidad Carlos III de Madrid}
    }
  \maketitle
} \fi

\if1\blind
{
  \bigskip
  \bigskip
  \bigskip
  \begin{center}
    {\LARGE\bf Title}
\end{center}
  \medskip
} \fi

\bigskip
\begin{abstract}
B-spline regression constitutes a widely used framework for nonparametric modeling. The performance of this methodology depends on specifying the number and placement of changepoints, known as knots, prior to the estimation process.  Such knot sequence determines the dimension of the B-spline basis used to represent the regression function and the number of coefficients to be estimated. Therefore, the knots' choice affects the model's flexibility, influencing its smoothness and goodness-of-fit. Traditionally, this problem has been addressed either by explicitly selecting knots, via knot-selection algorithms, or by regularization methods, such as P-splines, which automatically tune the regressor's smoothness. The latter have become the standard in generalized additive models (GAMs). In contrast, knot-selection techniques, frequently neglected because of computational or modeling limitations, provide certain advantages which can be valuable in some contexts. 
In this work, we introduce a novel explicit knot-selection technique for GAMs based on an extension of the adaptive splines (A-splines) knot selection methodology, combined with a customized Fellner-Schall scheme for tuning the associated parameters. Our approach is evaluated on various synthetic and real datasets and compared with P-splines and state-of-the-art knot-selection techniques. The results indicate comparable performance, while producing models built on a substantially smaller number of basis elements.

\end{abstract}

\noindent%

\section{Introduction}
\label{sec:1}

Modeling complex phenomena through nonparametric methods is a central topic in contemporary statistics. Among the most widely used approaches, B-spline regression provides a flexible and interpretable framework for smooth function estimation and constitutes the foundation of many modern smoothing techniques. Such methodology requires specifying a set of changepoints, or knots, which determine the family of basis functions used for curve fitting, both in terms of their number and location \citep{de1978practical}. The resulting model is highly sensitive to this choice, as it directly controls the flexibility and smoothness of the fitted curve. An excessively large or poorly located set of knots may lead to overfitting, whereas an insufficient number of knots may oversmooth important features of the response, both of which would compromise the model's generalizability and derived insight. Consequently, knot selection remains an active research topic in several settings, including univariate regression \citep{goepp2025spline}, additive models \citep{kim2026simultaneous, thielmann2025enhancing}, generalized regression models \citep{Dimitrova2023}, and functional data smoothing \citep{magistris2025adaptive}.

In practice, smoothing has mainly been addressed either through explicit knot-selection or regularization techniques. The former seeks to identify a parsimonious set of knots capable of adequately representing the underlying response structure.  In contrast, regularization approaches fix a relatively large set of knots and control smoothness through a penalty term on the spline coefficients. On the one hand, penalized splines (P-splines) \citep{eilers1996flexible}, which combine B-spline regression with a smoothness penalty, have become one of the standard approaches for smooth regression. In this framework, smoothing is reduced to selecting appropriate penalty parameters, typically through generalized cross-validation or automatic procedures such as the Fellner--Schall algorithm \citep{wood2017generalized}. Knot-selection techniques, on the other hand, offer certain advantages over regularization approaches (\cite{goepp2025spline}). Namely, (i) the knots provide contextual insight into the response's nature by identifying regions in which the response exhibits behavioral changes, (ii) they showcase adaptability to curvature or abrupt changes in its behavior, and (iii) they yield models that are built on a reduced number of basis functions. 

There are contexts in which explicitly seeking the simplest model is key for further applicability, and it is precisely this idea which motivates the present work: leveraging simplicity in terms of the number of basis elements, which is what we refer to as model \textit{sparsity}, with predictive performance. For instance, within surrogate modeling , estimating functions with a low number of parameters is convenient when later embedding them as either the objective function or constraints in an optimization model, since the dimension of the surrogate optimization problem critically conditions the computational requirements to solve it \citep{cuesta2025}. Another context in which this notion of sparsity might be key is numerical methods, such as in the Weighted Isogeometric Collocation based on Spline Projectors method \citep{GIUST2022114554}: {``[\ldots] the choice of the collocation nodes is crucial to ensure stability and good approximation properties [\ldots]''}. These examples illustrate the principal advantage provided by the presented paradigm: using sparsely estimated models, that is, through a low number of basis elements, potentially alleviates subsequent computational demand.

A variety of methods in the literature aim at performing knot selection in univariate regression. On the one hand, among the most prominent fixed-set knot selection methods, that is, those selecting knots from a predefined set, \cite{osborne1998knot,yuan2013adaptive} propose LASSO penalization, whereby knots associated with non-significant basis functions are removed. Another well-known approach is the Multivariate Adaptive Regression Splines (MARS) method of \cite{friedman1991multivariate}, which employs a stepwise forward/backward knot selection procedure. On the other hand, a different paradigm is that of free-knot splines, in which the search space of the knots is not constrained to a fixed set. Among them, \cite{denison1998automatic, green1995reversible} propose the use of a Bayesian approach, the Bayesian Adaptive Regression Splines (BARS), which employs a Markov Chain Monte Carlo procedure to estimate the model's coefficients and knots. Additionally, several purely heuristic optimization procedures have also been proposed in \cite{jupp1978approximation, spiriti2013knot}.

When extending the univariate regression setting to additive models, the smoothing problem must be addressed across all covariates simultaneously. This leads to a considerably more challenging estimation problem, as the estimated effect of each covariate depends on the simultaneous estimation of the remaining additive components. The aforementioned knot selection algorithms have been adapted to this setting: a LASSO-based approach along variable selection is proposed in \cite{kim2026simultaneous}; MARS and BARS have already been developed within the additive framework; and \cite{thielmann2025enhancing} introduces a heuristic free-knot approach. 
Most of them, however, present computational limitations arising from the \textit{curse of dimensionality}: simultaneously selecting sets of knots and/or parameters across multiple dimensions substantially increases the computational burden. As a result, their extension to additive models has attracted limited attention, and P-splines together with the Fellner--Schall algorithm for penalty parameter tuning have become the preferred method for fitting additive models \citep{de1978practical, mgcvR}.

Further, while P-splines have been successfully extended to generalized additive models (GAMs), there is currently no widely adopted knot-selection algorithm for this setting. Although several attempts have been made, they present important limitations. For example, the method of \cite{Dimitrova2023} relies on a backfitting algorithm, which does not estimate the covariate effects simultaneously and is restricted to B-spline bases of degrees 1, 2, and 3. Likewise, \cite{thielmann2025enhancing} discusses the problem but does not provide a method tailored to non-Gaussian responses.

Three knot-selection approaches are particularly relevant to develop our methodology: (i)~Adaptive splines (A-splines), proposed by \cite{goepp2025spline} for univariate non-linear regression, are a fixed-set knot-selection method. They start from a large set of initially equidistant knots and remove the least relevant ones using a surrogate of the $L_0$ pseudonorm penalty (\cite{frommlet2016adaptive, rippe2012visualization}). This approach has been shown to be considerably simpler and significantly faster than several previously proposed methods while achieving comparable predictive performance; (ii)  Geometrically Designed least squares splines with variable knots (GeDS), a free-knot paradigm apt for generalized non-linear models and single-term bivariate models developed in \cite{Dimitrova2023}. Their approach relies on a backfitting two-stage procedure based on Schoenberg's variation diminishing spline approximations.  Although applicable to GAM fitting, it inherits the limitations of backfitting \citep{hastie1986generalized}, as the covariate effects are not estimated simultaneously, which may compromise the quality of the resulting models; and (iii) The work of \cite{thielmann2025enhancing} proposes a free-knot approach that combines knot selection and penalty smoothing within the additive setting to achieve adaptive smoothing. Specifically, the sets of free knots and the penalty parameters are jointly optimized using the Particle Swarm Optimization (PSO) heuristic. Unlike the approach we develop in this paper, which aims to obtain sparse models, their framework performs knot selection to achieve adaptive smoothing without introducing numerous penalty parameters, as it is traditionally done in adaptive P-splines \citep{rodriguez2019estimation}.

An early attempt to extend A-splines beyond univariate regression can be found in \cite{seck2022adaptive}, where the authors develop an approach for Binomial and Poisson GAMs  based on piecewise-constant (degree-0) B-spline bases and a backfitting procedure, and illustrate it with two actuarial applications. While this represents an important first step, its applicability is limited to a narrow class of GAMs, motivating the development of a more general methodology.


In this work, we propose a new framework for automatic knot selection in generalized additive models: \textit{Automatic Knot Selection in Smooth Additive Models} (AKSSAM). Our methodology extends the A-splines paradigm to additive and generalized additive settings by combining explicit knot selection with automatic penalty parameter estimation through a tailored implementation of the Fellner--Schall algorithm. More specifically, knot selection is performed via an $L_0$-pseudonorm quadratic surrogate penalty approach, which is based on simultaneously and automatically optimizing the penalty parameters and the surrogate terms. Such task is achieved by developing a customized  Fellner--Schall scheme within an alternating optimization procedure. The resulting framework avoids multidimensional grid-search strategies for penalty parameter tuning and yields sparse smooth additive models in a computationally tractable manner. 
The proposed methodology therefore extends automatic knot-selection procedures from univariate regression settings to GAMs, providing an alternative to regularization-based GAM fitting when sparse and interpretable spline representations are desirable.

The remainder of his paper is organized as follows: Section \ref{sec:2} outlines B-spline regression and penalty parameter selection. In Section \ref{sec:3},  the proposed algorithm for automatic knot selection in GAMs, namely AKSSAM, is presented. Section \ref{sec:4} reports the computational experiments, including both simulation studies and real-data applications. Finally, Section \ref{sec:5} presents the conclusions and discusses directions for future research.

\section{Background and notation} \label{sec:2}
This section briefly reviews B-spline regression and the role of knot placement in smooth curve estimation (Section~\ref{subsec:2.1}), followed by P-splines and smoothing parameter estimation via the Fellner--Schall algorithm (Section~\ref{subsec:2.2}).


\subsection{B-spline regression} \label{subsec:2.1}


Consider the univariate regression model with continuous covariate $X$ and response $Y$,
\begin{equation}\label{eq:univariate non-linear model}
	y_i = f(x_i) + \varepsilon_i,\qquad i = 1,\ldots,n,
\end{equation}
where $f$ is an unknown smooth function, $\{(x_i,y_i)\}_{i=1}^n$ is the set of observations and $\varepsilon_i$ are independent Gaussian errors with zero mean. A common approach to model $f$ in \eqref{eq:univariate non-linear model} is  B-spline regression, which approximates $f$ by a linear combination of B-spline basis functions \citep{de1978practical}.

Given a non-decreasing sequence of real numbers $\mathbf{t} = \{t_1,\ldots, t_k\}$, a {B-spline} of degree $q \geq 0$,  denoted by $B_{q,l}$, is a piecewise polynomial function of degree $q$ recursively defined as follows:

$$ B_{0,l}(x) = \begin{cases}1, & \text{if}\; x\in [t_l,t_{l+1})\\ 0 & \text{otherwise}\end{cases}$$and
\begin{equation*}
	B_{q,l}(x) = \frac{x-t_l}{t_{l+q}-t_{l}}B_{q-1,l}(x) + \frac{t_{l+q+1}-x}{t_{l+q+1}-t_{l+1}}B_{q-1, l+1}(x),\;\textrm{if } q>0,
\end{equation*}
for $l = 1, \ldots, k-q-1$. The resulting basis functions $B_{q,l}$ have compact local support and are fully determined by the knots $\mathbf{t}$ and the spline degree $q$. 

From this point onward, we denote the B-spline basis functions simply by $\{B_l\}_{l=1}^{k-q-1}$, omitting the degree whenever no ambiguity arises. Let $x_{\min}$ and $x_{\max}$ denote the minimum and maximum observed values of $X$, respectively. We assume that the knot sequence $\mathbf{t}$ contains no repeated knots and is chosen so that $x_{\min}$ and $x_{\max}$ are the boundary knots. In addition, it contains $k-2q-2$ internal knots, collected in the vector $\mathbf{t}^*$, together with $q$ external knots to the left of $x_{\min}$ and $q$ external knots to the right of $x_{\max}$, yielding a total of $k$ knots. This knot configuration improves the approximation near the boundaries, ensuring that the observed covariate domain is covered with enough B-spline functions, and enables extrapolation beyond the observed covariate range.

Figure~\ref{fig: BBasis} shows two B-spline bases defined on the interval $[0,1]$ (i.e., $x_{\min}=0$ and $x_{\max}=1$) using different knot sequences: equally spaced (Figure~\ref{fig:BBasis_a}) and non-equally spaced (Figure~\ref{fig:BBasis_b}).


\begin{figure}[htbp]
	\centering
    
    \begin{subfigure}[b]{0.45\textwidth}
        \centering
        \includegraphics[width=\textwidth]{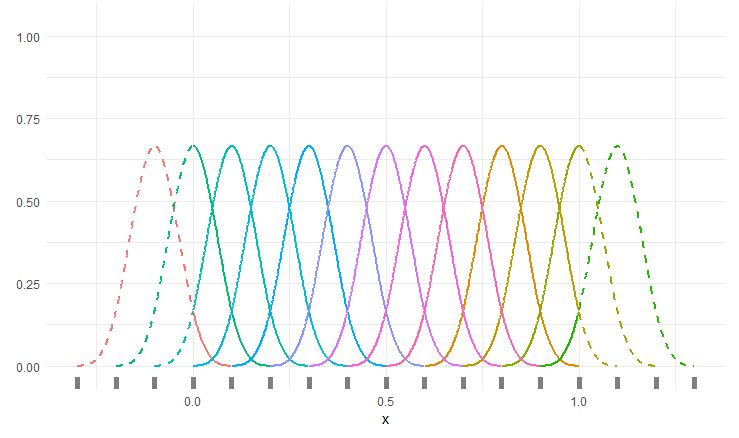}
        \caption{{B-spline} basis with {equally spaced} knots with $q = 3$.}
        \label{fig:BBasis_a}
    \end{subfigure}
    \hfill
    \begin{subfigure}[b]{0.45\textwidth}
        \centering
        \includegraphics[width=\textwidth]{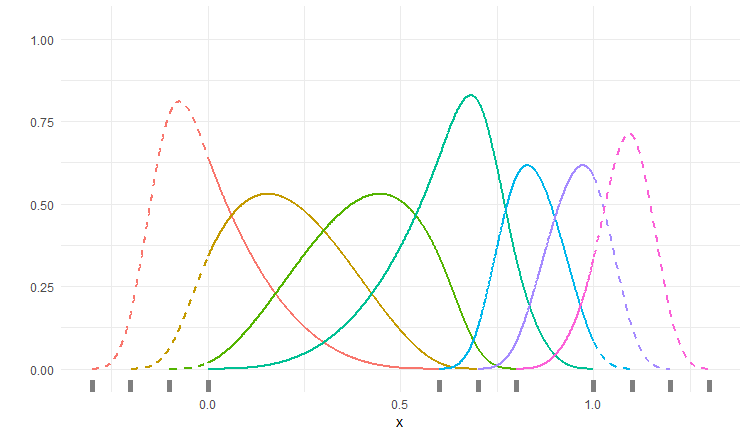}
        \caption{B-spline basis with non-equally spaced knots with $q = 3$.}
        \label{fig:BBasis_b}
    \end{subfigure}
    
    \caption{B-spline bases defined in $[0,1]$ with knots indicated in gray on the $x-$axis. Solid lines represent the basis functions within the covariate domain, while dashed lines indicate the extrapolation region.}
    \label{fig: BBasis}
\end{figure}

Then, function $f$ in \eqref{eq:univariate non-linear model} can be approximated in the interval $[x_{\min},x_{\max}]$ by a curve $S$ that is a linear combination of the B-spline basis functions $\{B_l\}_{l=1}^{k-q-1}$ resulting from the knot sequence $\mathbf{t}=\{t_1,\ldots,t_{k}$\}, yielding:

\begin{equation}\label{eq: Univariate Smoothing function}
	f(x)\approx S(x)=\sum_{l=1}^{k-q-1}\alpha_l B_l(x), \qquad x\in[x_{\min},x_{\max}],
\end{equation}
where $\alpha_l,$ $l=1,\ldots,k-q-1,$ are the regression coefficients. 

To estimate $S$ in \eqref{eq: Univariate Smoothing function} the least squares criterion is commonly used. Let $\mathbf{y}=(y_1,\ldots,y_n)'$ be the vector containing the observed values
of the response variable and $\mathbf{B}=[B_l(x_i)]_{il}$ denote the design matrix. The  estimated regression coefficients $\hat{\boldsymbol{\alpha}}=(\hat{\alpha}_1,\ldots,\hat{\alpha}_{k-q-1})'$  are obtained as:
\begin{equation}\label{eq: Univariate unpenalized Problem}
	\hat{\boldsymbol{\alpha}}
	\in
	\arg\min_{\boldsymbol{\alpha}\in\mathbb{R}^{k-q-1}}
	\left\|
	\mathbf{y}-\mathbf{B}\boldsymbol{\alpha}
	\right\|^2
	=
(\mathbf{B}'\mathbf{B})^{-1}\mathbf{B}'\mathbf{y},
\end{equation}
where $\|\cdot\|$ denotes the Euclidean norm and $\boldsymbol{\alpha} = (\alpha_1,\ldots,\alpha_{k-q-1}).$

We point out that a key aspect of B-spline regression is the selection of the knot configuration, including both the number and the placement of the knots. An excessively large number of knots may lead to overfitting, whereas too few knots, or poorly positioned ones, may oversmooth important features of the response. Consequently, identifying an appropriate knot configuration is a central problem in B-spline smoothing.

\subsection{Penalized splines and smoothing parameter selection} \label{subsec:2.2}

To reduce the influence of the knot configuration on the shape of $S$ in \eqref{eq: Univariate Smoothing function}, penalized splines (P-splines) leverage a large number of knots and control the smoothness of the fitted function by adding a penalty to the regression coefficients $\boldsymbol{\alpha}$. This approach provides an alternative to explicit knot selection, as the curvature penalty prevents the overfitting that could otherwise arise from using a large number of knots while regulating the smoothness of the estimated model.

Several penalty terms have been proposed for P-splines. A common choice is to penalize the squared $d-$th order differences of the regression coefficients $\boldsymbol{\alpha}$, where $d$ is a non-negative integer.
That is, the terms $\Delta^d\alpha_j = \Delta(\Delta^{d-1}\alpha_j)$ for $j\geq d+1$, where $\Delta\alpha_j = \alpha_{j} - \alpha_{j-1}$. Let $\mathbf{D}_d$ denote the $d$-th order difference matrix and define $\mathbf{P}=\mathbf{D}_d'\mathbf{D}_d$. Then, the P-splines approach to estimate $S$ in \eqref{eq: Univariate Smoothing function} optimizes a penalized least squares criterion as follows:
\begin{equation}\label{eq:Differences penalty}
	\hat{\boldsymbol{\alpha}}
	\in
	\arg\min_{\boldsymbol{\alpha}\in\mathbb{R}^{k-q-1}}
	\left\|
	\mathbf{y}-\mathbf{B}\boldsymbol{\alpha}
	\right\|^2
	+
	\lambda\,\boldsymbol{\alpha}'\mathbf{P}\boldsymbol{\alpha}
	 = (\textbf{B}'\textbf{B} + \lambda \mathbf{P})^{-1}\textbf{B}'\mathbf{y},
\end{equation}
where $\lambda\geq0$ is the smoothing parameter controlling the roughness of the fitted curve.

The previous formulation can be naturally extended to generalized non-linear models, where the response variable is allowed to follow an exponential family distribution. In this setting, the target of modeling is no longer the response itself, but the linear predictor $\eta$, which is related to the conditional mean of the response through a link function $g$. The spline representation is therefore incorporated into the additive predictor as
\begin{equation}\label{eq: GLM}
	g(\mathbb{E}[Y|X])=\eta(X)=f(X)\approx S(X).
\end{equation}

Under this framework, estimation is performed through penalized likelihood maximization rather than  penalized least squares. Specifically, the regression coefficients are obtained as
\begin{equation}\label{eq: Penalized GLM Estimation}
	\hat{\boldsymbol{\alpha}}
	\in
	\arg\max_{\boldsymbol{\alpha}\in\mathbb{R}^{k-q-1}}
	\ell(\boldsymbol{\alpha})
	-
	\frac{\lambda}{2}
	\boldsymbol{\alpha}'\mathbf{P}\boldsymbol{\alpha}
	,
\end{equation}
where $\ell$ denotes the log-likelihood associated with the assumed response distribution and the penalty term controls the smoothness of the estimated cuve. Since the resulting optimization problem does not admit a closed-form solution, estimation is typically carried out through penalized iterative reweighted least squares (P-IRLS) \citep{nelder1972generalized}.

Assuming a sufficiently rich B-spline basis, that is, a large number of knots, the problem of controlling the smoothness of $S$ in P-splines is effectively reduced to selecting an appropriate smoothing parameter $\lambda$. Classical approaches determine $\lambda$ by optimizing criteria such as generalized cross-validation (GCV), typically optimized through a grid search over the parameter space. In additive models, however, the presence of multiple smooth components requires the simultaneous selection of several smoothing parameters, making these procedures computationally demanding.

An attractive alternative is the Fellner--Schall algorithm \citep{wood2017generalized}, which automatically estimates the smoothing parameter(s) by maximizing an approximate restricted marginal likelihood. Taking $\mathbf{P}_\lambda = \lambda\,\mathbf{P} = \lambda\, \mathbf{D}_d'\mathbf{D}_d$ yields the iterative update
\begin{equation}\label{eq: Fellner Gaussian Update - 1D}
	\lambda^* =
	\widehat{\sigma}_{\lambda}^2
	\,
	\frac{
		\text{tr}\left[\mathbf{P}_\lambda^-\mathbf{P}\right]
		-
		\text{tr}\left[
		\left(\mathbf{B}'\mathbf{B}+\mathbf{P}_\lambda\right)^{-1}\mathbf{P}
		\right]
	}{
		\hat{\boldsymbol{\alpha}}_{\lambda}' \mathbf{P} \hat{\boldsymbol{\alpha}}_{\lambda}
	}
	\,\lambda,
\end{equation}
where $\mathbf{P}_\lambda^-$ denotes the Moore--Penrose pseudoinverse of $\mathbf{P}_\lambda$, $\hat{\boldsymbol{\alpha}}_\lambda$ are the regression coefficient estimates obtained from \eqref{eq:Differences penalty} for a fixed $\lambda$, and
\begin{equation}\label{eq: sigma estimation}
	\widehat\sigma^2_{\lambda} = \frac{\left\|\mathbf{y} - \mathbf{B}\hat{\boldsymbol{\alpha}}_{\lambda}\right\|^2}{n - \text{tr}\left[\left(\mathbf{B}'\mathbf{B} + \mathbf{P}_\lambda\right)^{-1}\mathbf{B}'\mathbf{B} \right]}
\end{equation}
is the residual variance estimator with effective-degrees-of-freedom correction.

The main advantage of the Fellner--Schall approach is that smoothing parameter estimation is performed automatically and scales efficiently to additive settings involving multiple smoothing parameters. This property plays a central role in the methodology we develop in Section~\ref{sec:3}.

\begin{remark}
	For non-Gaussian responses, the Fellner--Schall update becomes
	\begin{equation}\label{eq: Fellner General Update - 1D}
		\lambda^*
		=
		\phi
		\,
		\frac{
			\text{tr}\left[\mathbf{P}_\lambda^-\mathbf{P}\right]
			-
			\text{tr}\left[
			\left(\mathbf{B}'\Omega\mathbf{B}+\mathbf{P}_\lambda\right)^{-1}\mathbf{P}
			\right]
		}{
			\hat{\boldsymbol{\alpha}}_\lambda' \mathbf{P} \hat{\boldsymbol{\alpha}}_\lambda
		}
		\,\lambda,
	\end{equation}
	where $\Omega$ denotes the diagonal weight matrix obtained from the P-IRLS algorithm, $\phi$ is the scale parameter of the exponential family distribution { and $\hat{\boldsymbol{\alpha}}_\lambda$ is the solution of \eqref{eq: Penalized GLM Estimation}}.

\end{remark}

\section{AKSSAM: Automatic Knot Selection in Smooth Additive Models}
\label{sec:3}

This section introduces AKSSAM, our proposed knot-selection framework for generalized additive models (GAMs). Building upon the adaptive spline (A-spline) methodology of \citet{goepp2025spline}, we extend it from the univariate setting to additive models (AMs) and develop an algorithm that jointly estimates the model's coefficients and tuning parameters. Section~\ref{subsec:3.1} briefly reviews A-splines, whose $L_0-$penalty surrogate provides the foundation of the proposed approach. Section~\ref{subsec:3.2} generalizes this methodology to the additive setting, while Section~\ref{subsec:3.3} presents the tuning parameter strategy we follow for the additive extension of A-splines. Finally, Section~\ref{subsec:3.4} formulates AKSSAM through an alternating optimization scheme that integrates the Fellner--Schall algorithm for the automatic estimation of the penalty parameters, both for AMs and GAMs.


\subsection{A-splines: Adaptive splines}\label{subsec:3.1}

A-splines, introduced by \citet{goepp2025spline}, address the explicit knot selection problem in univariate nonlinear regression. The method formulates knot selection as a penalized regression problem in which an $L_0-$pseudonorm penalty is used to identify and remove the least relevant knots from an initially dense knot sequence. Since the resulting optimization problem is computationally challenging, the $L_0-$penalty is replaced by the adaptive ridge quadratic surrogate \citep{frommlet2016adaptive,rippe2012visualization}, and it is then solved using an alternating scheme.


The theoretical justification for this knot-selection strategy is established in \citet[Section~2.2.1]{goepp2025spline}. Consider a B-spline basis of degree $q$ defined over an equally spaced knot sequence $\mathbf{t}=\{t_1,\ldots,t_k\}$. If there exists an index $l^*\geq q+2$ such that $\Delta^{q+1}\alpha_{l^*}=0$, then the B-spline basis can be equivalently reparametrized over the reduced knot sequence $\mathbf{t}' = \{t_1,\ldots,t_{l^*-1}, t_{l^*+1},\ldots, t_k\},$ obtained by removing the knot $t_{l^*}$. Consequently, whenever the $(q+1)-$th order difference associated with an interior knot vanishes, eliminating that knot and recomputing the B-spline basis produces exactly the same fitted function as the one obtained using the original knot sequence. This result provides the theoretical foundation for identifying and removing unnecessary knots.


Therefore, under the notation and assumptions of Section~\ref{subsec:2.1}, the A-spline approach solves the following optimization problem:

\begin{equation}\label{eq: L0 penalized problem univariate}
\hat{\boldsymbol{\alpha}} \in \arg\min_{\boldsymbol{\alpha}\in\mathbb{R}^{k-q-1}} \left\|\mathbf{y} - \mathbf{B}\boldsymbol{\alpha}\right\|^2+ {\lambda} \sum_{l = q+2}^{k-q-1} \left\|\Delta^{q+1}\alpha_{l}\right\|_0,
\end{equation}
where the penalty parameter $\lambda>0$ controls the influence of the penalty in the objective function. The $L_0-$pseudonorm is adopted because it penalizes only the presence of non-zero terms, irrespective of their magnitude. Consequently, all non-zero $(q+1)-$th order differences are penalized equally, promoting sparse solutions in which many of these differences are estimated as exactly zero. By the theoretical result described above, the corresponding knots can then be removed without altering the fitted curve.

Although formulation \eqref{eq: L0 penalized problem univariate} directly promotes sparse knot configurations, it is computationally intractable due to the non-differentiability and non-convexity of the $L_0-$pseudonorm. To overcome the former issue, the $L_0-$penalty is approximated by the adaptive ridge quadratic surrogate, resulting in a differentiable optimization problem that can be solved through an alternating optimization scheme. In other words,
taking $\epsilon = 10^{-5}$ and $\omega_l = (\left(\Delta^{q+1}\alpha_{q+1+l}\right)^2 + \epsilon^2)^{-1}$ for $l = 1,\ldots, k-2q-2$, then
\[
\left\| \Delta^{q+1}\alpha_{q+1+l} \right\|_{0} \approx \omega_l(\Delta^{q+1}\alpha_{q+1+l})^2.
\]
Using this surrogate, the optimization problem in \eqref{eq: L0 penalized problem univariate} can be approximated by:
\begin{equation}\label{eq: Penalized Gaussian L0 scenario}
\arg\min_{\boldsymbol{\alpha}\in\mathbb{R}^{k-q-1}} \left\|\mathbf{y}  - \mathbf{B}\boldsymbol{\alpha}\right\|^2 + \boldsymbol{\alpha}'\mathbf{P}^{\lambda\boldsymbol{\omega}}\boldsymbol{\alpha},
\end{equation}
where $\mathbf{P}^{\lambda\boldsymbol{\omega}} = \lambda \mathbf{D}'_{q+1} W \mathbf{D}_{q+1}$, $\boldsymbol\omega = (\omega_{1},\ldots,\omega_{k-2q-2})$ and $W = \text{diag}(\boldsymbol\omega).$ 
However, problem \eqref{eq: Penalized Gaussian L0 scenario} does not admit a closed-form solution because the weights $\boldsymbol{\omega}$ depend on the regression coefficients $\boldsymbol{\alpha}$. If the weights are instead treated as fixed, the problem becomes quadratic in $\boldsymbol{\alpha}$, and its solution is given by
\begin{equation}\label{eq:Fixed-Weights solution}
\hat{\boldsymbol\alpha}_{\boldsymbol\omega} = \left(\mathbf{B}'\mathbf{B} + \mathbf{P}^{\lambda\boldsymbol{\omega}}\right)^{-1}\mathbf{B}'\mathbf{y}.
\end{equation}

\cite{goepp2025spline} solve problem~\eqref{eq: Penalized Gaussian L0 scenario}  through an alternating optimization scheme. Starting from an initial value of the weight vector $\boldsymbol{\omega}$, the regression coefficients are estimated by solving~\eqref{eq:Fixed-Weights solution} while keeping $\boldsymbol{\omega}$ fixed. The weights are then updated according to the estimated $(q+1)-$th order differences of the coefficients, and the resulting weights are used to re-estimate the regression coefficients. These two steps are repeated until convergence. The procedure yields the estimated regression coefficients $\hat{\boldsymbol{\alpha}}$ and weight vector $\hat{\boldsymbol{\omega}}$, from which knot selection is performed as follows:

\begin{equation}\label{eq: Approximated Knot selection}
    \mathbf{t}^{sel} = \left\{t_{q + 1+ l}\in\mathbf{t}^*\,|\, \|\Delta^{q+1}\hat\alpha_{q+1+l}\|_0\approx\hat\omega_l\left(\Delta^{q+1}\hat\alpha_{q+1+l}\right)^2 \approx 1,\; l =1,\ldots, k-2q-2\right\},
\end{equation}
that is, the interior knots whose associated surrogate $L_0$ term is estimated to be non-zero are retained.  The external and boundary knots of the original knot sequence $\mathbf{t}$ are added to the set \eqref{eq: Approximated Knot selection}, i.e., $\{t_1,\ldots,t_{q+1}\}\cup\{t_{k-q},\ldots,t_k\}$, and, afterwards the resulting B-spline basis is used to fit an unpenalized non-linear model, as in \eqref{eq: Univariate unpenalized Problem}.

The alternating optimization scheme assumes a fixed value of the penalty parameter $\lambda$.  \citet{goepp2025spline} tune $\lambda$  through a grid search using a modified Bayesian Information Criterion (BIC).


\subsection{Additive extension of A-splines}\label{subsec:3.2}

To establish the proposed methodology, we first review additive models (AMs), describing both their formulation and the identifiability constraints adopted in this work following \citet{wood2020inference}. We then show how the A-spline methodology can be extended from the univariate to the additive setting.


\subsubsection{Smooth Additive Models}
The {B-spline} regression approach in Section~\ref{subsec:2.1} can be naturally extended to the multivariate case. Let $X_1,\ldots,X_p$ be $p$ continuous covariates and $Y$ the response, consider the AM
\begin{equation}\label{eq:Additive model}
y_i = \alpha_0 + f_1({x_{1i}}) + \ldots + f_p(x_{pi}) + \varepsilon_i,\qquad i=1,\ldots,n,
\end{equation}
where $f_1,\ldots, f_p$ are unknown smooth functions, $\alpha_0$ the intercept, $\{(x_{1,i},\ldots,x_{pi},y_i)\}_{i=1}^n$ the set of observations and $\varepsilon_i$ are independent Gaussian errors with zero mean.


The smooth components of an AM are estimated by representing each function $f_r$ through a B-spline basis expansion. Specifically, for $r=1,\ldots,p$,  let $\mathbf{t}_r$ denote a sequence of $k_r$ knots, as in Section~\ref{subsec:2.1}, being  $\mathbf{t}^*_r$ the corresponding sequence of internal knots, that is defined taking as boundary knots $x_{r\min}$ and $x_{r\max}$, i.e., the minimum and maximum observed values of $X_r$, respectively. Then, using B-splines of degree $q_r$, the $r-$th smooth component is approximated as

\begin{equation}\label{eq:Multivariate Smoothing}
f_r(x_r) \approx S_r(x_r) = \sum_{l = 1}^{k_r - q_r-1} \alpha_{rl}\, B_{rl}(x_r),\qquad x_r\in[x_{r\min},x_{r\max}],
\end{equation}
where $\{B_{rl}\}_{l=1}^{k_r-q_r-1}$ denotes the corresponding B-spline basis, and $\alpha_{rl}$ are the regression coefficients, $ l=1,\ldots,k_r-q_r-1$, for $r=1,\ldots,p.$


Let $\mathbf{B}_{r} = \left[ B_{rl}(x_{ri}) \right]_{i,l},$ for $i = 1, \ldots, n$ and $l = 1, \ldots, k_r - q_r - 1$  be the design matrix for the basis expansion of each covariate $X_r$ and $\boldsymbol{\alpha}_r = (\alpha_{r1}, \ldots, \alpha_{r\,k_r - q_r - 1})'$ the associated regression coefficient vector, for $r=1,\ldots,p$. Denote by $\boldsymbol{\alpha} = (\alpha_0, \boldsymbol{\alpha}_1', \ldots, \boldsymbol{\alpha}_p')'$ the stacked coefficient vector and by $\mathbf{B} = \left[ \mathbf{1}^n \,|\, \mathbf{B}_1 \,|\, \ldots \,|\, \mathbf{B}_p \right]$ the full design matrix, of dimension $n\times m$ with $m= 1+\sum_{r=1}^p m_r$, and $m_r=k_r-q_r-1$ for each $r=1,\ldots,p$, which combines columnwise a column of ones of dimension $n$ and all covariate design matrices. Since each B-spline basis constitutes a partition of unity, $\mathbf{B}_r\mathbf{1}^{m_r} = \mathbf{1}^n$ for every $r$, the design matrix $\mathbf{B}$ is rank-deficient, and we impose that the effect of every covariate is centered, namely $\mathbf{1}^{n\prime}\mathbf{B}_r\boldsymbol{\alpha}_r = 0$ for $r=1,\ldots,p$, to ensure that the AM is identifiable. Note that \citet{wood2017generalized} take the positive definiteness of the
cross-product of the design matrix as a working hypothesis. This is justified in their framework because the identifiability constraints are absorbed into the model matrix by reparametrization, as described in \citet{wood2017generalizedbook}:  each smooth term is re-expressed in terms of one fewer, unconstrained coefficient, so that the centering condition holds by construction and the resulting design matrix has full column rank. Such a route is, however, not available in our case: the reparametrization alters the structure of $\mathbf{B}$ and, with it, that of the difference matrices $\mathbf{D}_{q_r+1}$, since the coefficients of the reparametrized
model are no longer B-spline coefficients. The knot-removal result reviewed in Section~\ref{subsec:3.1} and the adaptive
ridge penalty are both formulated in terms of the $(q_r+1)$-th order
differences of B-spline coefficients, so the A-spline machinery would no longer be justified under the transformed parameterization. We therefore preserve the original parameterization and enforce the constraints in their penalized form, following \citet{wood2020inference}, through the quadratic penalty \begin{equation}\label{eq:Identifiability penalty}
\mathbf{P}_{\text{I}} = \operatorname{diag}\left(0,\; (\mathbf{B}_1'\mathbf{1}^{n})(\mathbf{B}_1'\mathbf{1}^{n})',\;\ldots,\;(\mathbf{B}_p'\mathbf{1}^{n})(\mathbf{B}_p'\mathbf{1}^{n})'\right),
\end{equation}
so that $\boldsymbol{\alpha}'\mathbf{P}_{\text{I}}\boldsymbol{\alpha} = \sum_{r=1}^p \left(\mathbf{1}^{n\prime}\mathbf{B}_r\boldsymbol{\alpha}_r\right)^2$. Then, to estimate $S_r$ in \eqref{eq:Multivariate Smoothing}, $r=1,\ldots,p$, and thus the AM in \eqref{eq:Additive model}, the least squares criterion is used, yielding
\begin{equation}\label{eq:Addittive Solution}
\hat{\boldsymbol{\alpha}} \in \arg\min_{\boldsymbol{\alpha}\in\mathbb{R}^{m}} \left\|\mathbf{y} - \mathbf{B}\boldsymbol{\alpha}\right\|^2 + \boldsymbol{\alpha}'\mathbf{P}_{\text{I}}\boldsymbol{\alpha} = \left(\mathbf{B}'\mathbf{B} + \mathbf{P}_{\text{I}}\right)^{-1}\mathbf{B}'\mathbf{y},
\end{equation}
where the uniqueness of the solution, and thus the validity of the closed-form expression, follows from Lemma~\ref{lem:PD} below.

The choice of the knot sequences $\{\mathbf{t}_r\}_{r=1}^p$ remains a crucial aspect of model estimation. Since the smooth components are estimated simultaneously, the B-spline basis selected for one covariate may influence the estimation of the others, making the construction of appropriate knot sequences particularly challenging. This further motivates the use of regularization techniques or automatic knot-selection procedures.


As in the univariate setting, the P-spline framework extends naturally to AMs by augmenting problem \eqref{eq:Addittive Solution} with a roughness penalty for each smooth component. Accordingly, model complexity is controlled through a collection of smoothing parameters, one for each additive term. These parameters can be estimated automatically using the Fellner--Schall algorithm described in Section~\ref{subsec:2.2}, which jointly updates the smoothing penalties while estimating the AM.

\subsubsection{Extending A-splines}

Building upon the AM framework introduced above, we extend the A-spline methodology to the additive setting by simultaneously performing knot selection for all smooth components. This is achieved by incorporating an $L_0-$based penalty for each additive term, allowing the knot sequence associated with every covariate to be selected automatically.

Following the notation of Section~\ref{subsec:3.1}, let $\boldsymbol{\omega}_r=(\omega_{r1},\ldots,\omega_{r\,k_r-2q_r-2})^\prime$ denote the adaptive ridge weight vector for the $r-$th smooth component, where
$\omega_{rl}=\left((\Delta^{q+1}\alpha_{r\,q_r+1+l})^2+\epsilon^2\right)^{-1},$ $l=1,\ldots,k_r-q_r-2$. Let $W_r=\operatorname{diag}(\boldsymbol{\omega}_r)$ be the corresponding diagonal weight matrix, and define the joint weight vector $\boldsymbol{\omega}=(\boldsymbol{\omega}_1^\prime,\ldots,\boldsymbol{\omega}_p^\prime)^\prime$, whose dimension is $s=\sum_{r=1}^{p}(k_r-2q_r-2)$. Furthermore, let $\lambda_r,\, r=1,\ldots,p,$ denote the penalty parameter associated with the $r-$th smooth component, and define $\boldsymbol{\lambda}=(\lambda_1,\ldots,\lambda_p)$ as the vector collecting all penalty parameters.

For each $r=1,\ldots,p$, let $\mathbf{P}^{\boldsymbol{\omega}_r} = \mathbf{D}_{q_r+1}'W_r\mathbf{D}_{q_r+1}$ denote the adaptive ridge penalty matrix of the $r$-th component with unit penalty parameter, and 
let $\mathbb{P}^{\boldsymbol{\omega}_r}$ denote the $m\times m$ block-diagonal matrix with $\mathbf{P}^{\boldsymbol{\omega}_r}$ in the diagonal block corresponding to the coefficients of the $r-$th covariate and zeros in all remaining blocks.
The joint adaptive ridge penalty is then
\begin{equation}\label{eq: Additive A-splines i}
\mathbf{P}^{\boldsymbol{\lambda}\boldsymbol{\omega}} = \sum_{r=1}^p \lambda_r\,\mathbb{P}^{\boldsymbol{\omega}_r},
\end{equation}
and we write
\begin{equation}\label{eq: Additive A-splines ii}
\mathbf{P}^{\boldsymbol{\lambda}\boldsymbol{\omega}}_{\text{I}} = \mathbf{P}^{\boldsymbol{\lambda}\boldsymbol{\omega}} + \mathbf{P}_{\text{I}}.
\end{equation} 
Therefore, the A-splines problem extended to the AM setting is given by
\begin{equation}\label{eq: AM A-Splines Problem}
\arg\min_{\boldsymbol{\alpha}\in\mathbb{R}^{m}} \left\|\mathbf{y}-\mathbf{B}\boldsymbol{\alpha}\right\|^2 + \boldsymbol{\alpha}' \mathbf{P}^{\boldsymbol{\lambda}\boldsymbol{\omega}}_{\text{I}}\boldsymbol{\alpha}.
\end{equation}
As in the univariate setting, problem~\eqref{eq: AM A-Splines Problem} does not admit a closed-form solution because the adaptive ridge weights $\boldsymbol{\omega}$ depend on the regression coefficients $\boldsymbol{\alpha}$. Consequently, the problem is solved using the same alternating optimization scheme described in Section~\ref{subsec:3.1}, alternating between updating the regression coefficients and the adaptive ridge weights. In particular, when the weight vector $\boldsymbol{\omega}$ is held fixed, the optimization problem becomes quadratic in $\boldsymbol{\alpha}$, and its analytical solution is
\begin{equation}\label{eq: AM A-Splines Problem solution}
\hat{\boldsymbol{\alpha}}_{\boldsymbol\omega} = \left(\mathbf{B}'\mathbf{B}+\mathbf{P}^{\boldsymbol{\lambda}\boldsymbol{\omega}}_{\text{I}}\right)^{-1}\mathbf{B}'\mathbf{y}.
\end{equation}
Then, at convergence,
the knot-trimming procedure in \eqref{eq: Approximated Knot selection} can be applied to perform knot selection in each covariate individually.


In contrast to the univariate setting, the additive version of A-splines associates a distinct penalty parameter $\lambda_r$ with each smooth component, $r=1,\ldots,p,$ allowing the sparsity of each knot sequence to be controlled independently. Following the strategy of \citet{goepp2025spline}, these $p$ penalty parameters could be selected through a grid search based on a suitable model selection criterion. However, such an approach quickly becomes computationally impractical as the search space grows exponentially with the number of smooth components. This limitation motivates the development of an automatic procedure for the simultaneous estimation of all penalty parameters, which is presented in the following section.

\subsection{Tuning parameters in additive A-splines: a Fellner--Schall approach}\label{subsec:3.3}


To address penalty tuning in additive A-splines, we propose a tailored implementation of the Fellner--Schall algorithm that automatically optimizes the penalty parameters in problem~\eqref{eq: AM A-Splines Problem}, thereby avoiding computationally expensive grid search procedures. Because the adaptive ridge weights depend on the regression coefficients, the algorithm cannot be applied directly. We therefore begin by considering the problem with fixed adaptive ridge weights. This section proves that the assumptions of the Fellner--Schall algorithm hold under this formulation, laying the foundation for the algorithm developed in Section \ref{subsec:3.4}.

A key observation is that the identifiability penalty $ \mathbf{P}_{\text{I}}$ is not a regularization in the usual sense. It introduces no tuning parameter and is included solely to enforce the identifiability constraint.
Accordingly, when applying the Fellner--Schall methodology, the quadratic form $\boldsymbol{\alpha}'\mathbf{P}_{\text{I}}\boldsymbol{\alpha}$ is treated as part of the fitting criterion, together with the residual sum of squares, rather than as part of the smoothing penalty $\mathbf{P}^{\boldsymbol{\lambda}\boldsymbol{\omega}}$. This distinction is essential: Theorem~1 of \citet{wood2017generalized} requires the positive definiteness of the matrix that remains fixed as the penalty parameters vary  ($\mathbf{B}'\mathbf{B}$ in their formulation) and this condition fails in our setting, since $\mathbf{B}$ is rank-deficient by construction. Grouping $\mathbf{P}_{\text{I}}$ with the fitting criterion replaces this matrix by $\mathbf{B}'\mathbf{B}+\mathbf{P}_{\text{I}}$, which is shown to be positive definite in Lemma~\ref{lem:PD} below.

We first state the structural assumption under which identifiability is guaranteed.

\begin{assumption}\label{ass:rank}
	Each design matrix $\mathbf{B}_r$, $r = 1,\ldots,p$, has full column rank, and the only linear dependencies among the columns of $\mathbf{B}$ are those induced by the partition-of-unity property of the B-spline bases; that is,
	\[
	\text{Null}(\mathbf{B}) = \operatorname{span}\{\mathbf{u}_1,\ldots,\mathbf{u}_p\},
	\]
	where $\mathbf{u}_r\in\mathbb{R}^m$ denotes the vector whose first entry is $-1$, whose $r$-th block is $\mathbf{1}^{m_r}$ and whose remaining entries are zero.
\end{assumption}

Note that each $\mathbf{u}_r$ indeed belongs to $\text{Null}(\mathbf{B})$, since $\mathbf{B}\mathbf{u}_r = -\mathbf{1}^n + \mathbf{B}_r\mathbf{1}^{m_r} = \mathbf{0}$ by the partition-of-unity property: adding a constant to one smooth component while subtracting it from the intercept leaves the fitted model unchanged. Assumption~\ref{ass:rank} states that these  dependencies are the only ones, and it holds generically, provided that each covariate takes at least $m_r$ distinct values and that no exact concurvity relationship exists among the covariates.

\begin{lemma}\label{lem:PD}
	Under Assumption~\ref{ass:rank}, the matrix $\mathbf{B}'\mathbf{B} + \mathbf{P}_{\text{I}}$ is positive definite.
\end{lemma}
\begin{proof}
	Both $\mathbf{B}'\mathbf{B}$ and $\mathbf{P}_{\text{I}}$ are positive semidefinite, hence so is their sum, and for positive semidefinite matrices $\text{Null}(\mathbf{B}'\mathbf{B} + \mathbf{P}_{\text{I}}) = \text{Null}(\mathbf{B}'\mathbf{B})\cap\text{Null}(\mathbf{P}_{\text{I}}) = \text{Null}(\mathbf{B})\cap\text{Null}(\mathbf{P}_{\text{I}})$. Let $\boldsymbol{\alpha}$ belong to this intersection. By Assumption~\ref{ass:rank}, $\boldsymbol{\alpha} = \sum_{r=1}^p c_r\mathbf{u}_r$ for some $c_1,\ldots,c_p\in\mathbb{R}$; in particular, $\boldsymbol{\alpha}_r = c_r\mathbf{1}^{m_r}$ for every $r$ and $\alpha_0 = -\sum_{r=1}^p c_r$. Moreover, by \eqref{eq:Identifiability penalty}, $\boldsymbol{\alpha}\in\text{Null}(\mathbf{P}_{\text{I}})$ implies $\mathbf{1}^{n\prime}\mathbf{B}_r\boldsymbol{\alpha}_r = 0$. Using the partition-of-unity property, $\mathbf{1}^{n\prime}\mathbf{B}_r\boldsymbol{\alpha}_r = c_r\,\mathbf{1}^{n\prime}\mathbf{B}_r\mathbf{1}^{m_r} = c_r\,\mathbf{1}^{n\prime}\mathbf{1}^n = n\,c_r$, so $c_r = 0$ for every $r$, and consequently $\alpha_0 = 0$. Hence $\boldsymbol{\alpha} = \mathbf{0}$ and the matrix is positive definite.
\end{proof}
\begin{remark}\label{rem:mixed model}
The grouping of $\mathbf{P}_{\text{I}}$ with the fitting criterion has a natural interpretation in terms of the mixed-model representation underlying the Fellner--Schall method \citep{wood2017generalized}. Since $\boldsymbol{\alpha}'\mathbf{P}_{\text{I}}\boldsymbol{\alpha} = \sum_{r=1}^p\left(\mathbf{1}^{n\prime}\mathbf{B}_r\boldsymbol{\alpha}_r\right)^2$, the identifiability penalty amounts to augmenting the sample with $p$ pseudo-observations, $0 = \mathbf{1}^{n\prime}\mathbf{B}_r\boldsymbol{\alpha}_r + \varepsilon^*_r$ with $\varepsilon^*_r\sim\mathcal{N}(0,\sigma^2)$, which softly enforce the centering of each smooth component. Stacking these $p$ rows below $\mathbf{B}$ yields an augmented design matrix $\widetilde{\mathbf{B}}$ satisfying $\widetilde{\mathbf{B}}'\widetilde{\mathbf{B}} = \mathbf{B}'\mathbf{B} + \mathbf{P}_{\text{I}}$, which is positive definite by Lemma~\ref{lem:PD}, and $\widetilde{\mathbf{B}}'\widetilde{\mathbf{y}} = \mathbf{B}'\mathbf{y}$, where $\widetilde{\mathbf{y}}$ appends $p$ zeroes to $\mathbf{y}$. The augmented model is thus exactly of the form considered by \citet{wood2017generalized}: a full-column-rank design combined with the improper Gaussian prior induced by the smoothing penalty alone, $\boldsymbol{\alpha}\sim\mathcal{N}\left(\mathbf{0},\sigma^2\left(\mathbf{P}^{\boldsymbol{\lambda}\boldsymbol{\omega}}\right)^-\right)$, which is the only term carrying the tuning parameters. Note that, for fixed $\boldsymbol{\lambda}$, the coefficient estimates are identical regardless of how $\mathbf{P}_{\text{I}}$ is grouped, since the penalized objective is unchanged; the grouping only determines the restricted marginal likelihood that the penalty-parameter update ascends.
\end{remark}

Given a fixed weight vector $\boldsymbol{\omega}$, by \eqref{eq: Additive A-splines i} and \eqref{eq: Additive A-splines ii} the smoothing penalty $\mathbf{P}^{\boldsymbol{\lambda}\boldsymbol{\omega}}$ is linear in the penalty parameters, with
\[
\frac{\partial}{\partial \lambda_r}\, \mathbf{P}^{\boldsymbol{\lambda}\boldsymbol{\omega}} = \mathbb{P}^{\boldsymbol{\omega}_r},\qquad r = 1,\ldots,p.
\]
The validity of the Fellner--Schall update in this setting rests on the following result.

\begin{theorem}\label{thm:AKSSAM THM}
    {Let Assumption~\ref{ass:rank} hold, let $\boldsymbol{\omega}$ be a fixed weight vector satisfying $\omega_{rl}>0$ for $l=1,\ldots,k_r-2q_r-2$ and $r=1,\ldots,p$ and let $\lambda_r>0$ for $r=1,\ldots,p$. }For matrices $\mathbf{B}$, $\mathbf{P}^{\boldsymbol{\lambda}\boldsymbol{\omega}}_{\text{I}}$ and $\mathbb{P}^{\boldsymbol{\omega_r}}$ as defined in this section, then
    \begin{equation}\label{eq: Proper AKSSAM justification}
    {\text{tr}\left[{\left(\mathbf{P}^{\boldsymbol{\lambda}\boldsymbol{\omega}}\right)}^-\mathbb{P}^{\boldsymbol{\omega}_r}\right]-\text{tr}\left[\left(\mathbf{B}'\mathbf{B} + \mathbf{P}^{\boldsymbol{\lambda}\boldsymbol{\omega}}_{\text{I}}\right)^{-1}\mathbb{P}^{\boldsymbol{\omega}_r}\right]>0.}
    \end{equation}
\end{theorem}
\begin{proof}
{We apply \citet[Theorem~1]{wood2017generalized}, whose statement involves a positive definite matrix, denoted $\mathbf{A}$ in what follows, and a positive semidefinite penalty $\mathbf{S}_{\boldsymbol{\lambda}}$, linear in $\boldsymbol{\lambda}$, with derivatives $\mathbf{S}_r$ and a null space independent of $\boldsymbol{\lambda}$. We take $\mathbf{A} := \mathbf{B}'\mathbf{B} + \mathbf{P}_{\text{I}}$, $\mathbf{S}_{\boldsymbol\lambda} := \mathbf{P}^{\boldsymbol{\lambda}\boldsymbol{\omega}}$ and $\mathbf{S}_r := \mathbb{P}^{\boldsymbol{\omega}_r}$, and verify the hypotheses:}
   \begin{enumerate}[label=(\roman*)]
   	\item {$\mathbf{A} = \mathbf{B}'\mathbf{B} + \mathbf{P}_{\text{I}}$ is positive definite, by Lemma~\ref{lem:PD}. }
    
   	\item { $\mathbf{P}^{\boldsymbol{\lambda}\boldsymbol{\omega}}$ } is positive semidefinte, {} and linear in $\boldsymbol{\lambda}$, with $\partial\mathbf{P}^{\boldsymbol{\lambda}\boldsymbol{\omega}}/\partial\lambda_r = \mathbb{P}^{\boldsymbol{\omega}_r}$ positive semidefinite: each block is of the form $\mathbf{D}_{q_r+1}'W_r\mathbf{D}_{q_r+1}$ with $W_r$ a non-negative diagonal matrix. 
   
   	\item $\text{Null}\left(\mathbf{P}^{\boldsymbol{\lambda}\boldsymbol{\omega}}_{\text{I}}\right)$ is independent of $\boldsymbol\lambda$ (under the assumption that $\lambda_r>0,\, r=1,\ldots,p$): { since the addends in \eqref{eq: Additive A-splines i} are positive semidefinite, $\text{Null}\left(\sum_{r}\lambda_r\mathbb{P}^{\boldsymbol{\omega}_r}\right) = \bigcap_{r=1}^p\text{Null}\left(\mathbb{P}^{\boldsymbol{\omega}_r}\right)$ for any $\lambda_r>0,$  $r=1,\ldots,p$, which depends only on the (fixed) weights $\boldsymbol{\omega}$, while the penalty parameters act as scale factors.}
   \end{enumerate}
   { Noting that $\mathbf{A} + \mathbf{S}_{\boldsymbol{\lambda}} = \mathbf{B}'\mathbf{B} + \mathbf{P}^{\boldsymbol{\lambda}\boldsymbol{\omega}}_{\text{I}}$, } by \cite[Theorem~1]{wood2017generalized}, inequality \eqref{eq: Proper AKSSAM justification} is proven.
\end{proof}

Therefore, for a problem of the form in \eqref{eq: AM A-Splines Problem} where $\boldsymbol{\omega}$ is fixed, the Fellner--Schall update for penalty parameters $\boldsymbol\lambda$, is given by
\begin{equation}\label{eq: Fellner-Schall for AKSSAM}
\lambda^*_r = \widehat\sigma_{\boldsymbol\lambda}\frac{\text{tr}\left[\left(\mathbf{P}^{\boldsymbol{\lambda}\boldsymbol{\omega}}\right)^-\mathbb{P}^{\boldsymbol{\omega}_r}\right]-\text{tr}\left[\left(\mathbf{B}'\mathbf{B} + \mathbf{P}^{\boldsymbol{\lambda}\boldsymbol{\omega}}_{\text{I}}\right)^{-1}\mathbb{P}^{\boldsymbol{\omega}_r}\right]}{\boldsymbol{\hat{\alpha}_{\boldsymbol{\lambda}}}'\mathbb{P}^{\boldsymbol{\omega}_r} \boldsymbol{\hat{\alpha}_{\boldsymbol{\lambda}}}} \lambda_r,\quad r=1,\ldots,p,
\end{equation}
where $\boldsymbol{\hat{\alpha}_{\boldsymbol{\lambda}}}$ are given by \eqref{eq: AM A-Splines Problem solution} for fixed penalty parameters $\boldsymbol\lambda$, and where the associated unbiased variance estimator is 
\begin{equation}\label{eq:sigma}
\widehat{\sigma}^2_{\boldsymbol \lambda} = \frac{\left\| \mathbf{y} - \mathbf{B}\boldsymbol{\hat{\alpha}_{\boldsymbol\lambda}} \right\|^2}{n - \operatorname{tr} \left[ \left( \mathbf{B}' \mathbf{B} + \mathbf{P}^{\boldsymbol{\lambda}\boldsymbol{\omega}}_{\text{I}}  \right)^{-1} \mathbf{B}' \mathbf{B} \right]}.
\end{equation}
{\begin{remark}\label{rem:cheap trace}
	Since the adaptive ridge weights are strictly positive by construction, {\(\omega_{rl} \geq \quad\allowbreak \left(\max_l(\Delta^{q_r+1}\alpha_{r\,q_r+1+l})^2+\epsilon^2\right)^{-1} \allowbreak >0,\) $r=1,\ldots,p,\, l=1,\ldots,1,\ldots,k_r-2q_r-2,$ each block $\mathbf{P}^{\boldsymbol{\omega}_r}$} has rank $k_r-2q_r-2$, equal to the number of rows of $\mathbf{D}_{q_r+1}$. Moreover, the blocks $\{\mathbb{P}^{\boldsymbol{\omega}_r}\}_{r=1}^p$ act on disjoint sets of coefficients, so the first trace in \eqref{eq: Proper AKSSAM justification} admits the closed form
	\[
	\operatorname{tr}\left[\left(\mathbf{P}^{\boldsymbol{\lambda}\boldsymbol{\omega}}\right)^-\mathbb{P}^{\boldsymbol{\omega}_r}\right] = \frac{\operatorname{rank}\left(\mathbf{P}^{\boldsymbol{\omega}_r}\right)}{\lambda_r} = \frac{k_r-2q_r-2}{\lambda_r},
	\]
	 avoiding the explicit computation of the pseudoinverse.
\end{remark} }
Theorem~\ref{thm:AKSSAM THM} justifies that, if $\lambda_r\geq0$, then its Fellner-Schall update \eqref{eq: Fellner-Schall for AKSSAM} satisfies $\lambda_r^*\geq0$, ultimately rendering such algorithm valid for the setting in \eqref{eq: AM A-Splines Problem}, $r=1,\ldots,p$.

\subsection{AKSSAM algorithm}\label{subsec:3.4}


A natural way to address penalty parameter tuning in the additive A-splines problem \eqref{eq: AM A-Splines Problem} is to embed the Fellner–Schall scheme described in Section~\ref{subsec:3.3} within each iteration of the alternating optimization algorithm. This results in two nested loops: an outer loop that updates the weights until convergence of the $L_0$ surrogate, i.e. the adaptive ridge penalty, and an inner loop that optimizes the penalty parameters. The resulting algorithm, termed \textit{Automatic Knot Selection in Smooth Additive Models} (AKSSAM), performs automatic penalty parameter estimation for additive A-splines while avoiding the computational burden of grid-search procedures. Its implementation for additive models is presented in Section~\ref{subsubsec:3.4.1}, and Section~\ref{subsubsec:3.4.2} subsequently extends the approach to generalized additive models (GAMs).

\subsubsection{AKSSAM for AMs}\label{subsubsec:3.4.1}


Algorithm~\ref{Alg: AKSSAM} presents the pseudocode for AKSSAM. Starting from an initial set of penalty parameters, $\boldsymbol{\lambda}_0$, the algorithm initializes the weights as $\boldsymbol{\omega}=\boldsymbol{1}^s$ and computes an initial coefficient estimate as in \eqref{eq: AM A-Splines Problem solution}. It then executes the two nested loops described above: the inner loop updates the penalty parameters using the Fellner--Schall algorithm, while the outer loop uses the resulting coefficient estimates to update the weights of the adaptive ridge penalty until convergence. Once the algorithm has converged, the estimated weights $\hat{\boldsymbol{\omega}}$ and coefficients $\hat{\boldsymbol{\alpha}}$ are used to perform knot selection across the covariates according to \eqref{eq: Approximated Knot selection}. The selected internal knots associated with the $r$-th smooth component are given by
\[
    \mathbf{t}^{\text{sel}}_r =  \left\{t_{q_r+1+l}\in\mathbf{t}^*_r\,|\, \hat{\omega}_{rl}\left(\Delta^{q_r+1}\hat{\alpha}_{r\,q_r+1+l}\right)^2 \approx 1\right\}_{l = 1}^{k_r-2q_r-2},
\]
to which the external and boundary knots of the original knot sequence $\mathbf{t}_r$  are appended.

It is worth noting that, in Algorithm~\ref{Alg: AKSSAM}, the convergence criterion for the adaptive ridge procedure (i.e., the outer loop) is based on the relative change in the terms approximating the $L_0$ pseudonorms. Specifically, convergence is declared when these terms satisfy the relative tolerance Tol1 as:
\[\frac{\left\|\,\boldsymbol{\omega}^{\text{new}}\circ \left(\Delta^{\boldsymbol q+1}\boldsymbol{\alpha}^{\text{new}}\right)^2-\boldsymbol{\omega}\circ \left(\Delta^{\boldsymbol q+1}\boldsymbol{\alpha}\right)^2\,\right\|^2}{\left\|\,\boldsymbol{\omega}\circ \left(\Delta^{\boldsymbol q+1}\boldsymbol{\alpha}\right)^2\right\|^2}<\text{Tol1},
\]
where $\Delta^{\boldsymbol{q}+1}\boldsymbol{\alpha} = \left(\Delta^{q_1+1}\boldsymbol{\alpha}_1,\ldots,\Delta^{q_r+1}\boldsymbol{\alpha}_r\right)$ for each $r=1,\ldots,p$, $\circ$ represents the element-wise product of two vectors and $\boldsymbol{\omega}^{\text{new}},\boldsymbol{\alpha}^{\text{new}}$ represent the estimated weights and coefficients parting from $\boldsymbol{\omega},\boldsymbol{\alpha}$ within a single iteration of the loop. As well, if the iteration tolerance ItTol1 is violated, the loop is stopped.
Similarly, the convergence criterion for the Fellner--Schall algorithm (i.e., the inner loop) is based on the relative change in the estimated curves. Specifically, convergence is declared when this change falls below the tolerance Tol2 as:
\[\frac{\left\|\,\mathbf{B}(\boldsymbol{\alpha}^{\text{new}}-\tilde{\boldsymbol{\alpha}})\right\|^2}{\left\|\,\mathbf{B}\tilde{\boldsymbol{\alpha}}\right\|^2}<\text{Tol2},
\]
where $\boldsymbol{\alpha}^{\text{new}}$ are the estimated coefficients parting from $\tilde{\boldsymbol{\alpha}}$, and the loop is stopped if ItTol2 is exceeded.

\begin{algorithm}[h!]
\caption{AKSSAM for AMs}\label{Alg: AKSSAM}

\begin{flushleft}{\small
\textbf{Input:} Observations $\{(x_{i1},\ldots, x_{ip},y_i)\}_{i=1}^n$, knot sequences $\mathbf{t}_1,\ldots,\mathbf{t}_p$, initial penalty parameters $\boldsymbol{\lambda}_0$, tolerances $\text{ItTol1}, \text{Tol1}$, iteration and relative tolerances respectively for the adaptive ridge penalty and $\text{ItTol2}, \text{Tol2}$, iteration and relative tolerances respectively for the Fellner--Schall algorithm \\
\textbf{Output:} Sets of selected knots $\mathbf{t}_{1}^{\text{sel}},\ldots,\mathbf{t}_{p}^{\text{sel}}$ and the fitted coefficients ${\boldsymbol\alpha}^{\text{sel}}$ of  the additive model after knot selection}
\end{flushleft}

\begin{algorithmic}[1]
\State \textbf{Initialize} $\boldsymbol{\alpha}^{\text{new}} \gets \mathbf{0}$; $\boldsymbol{\omega}^{\text{new}} \gets \mathbf{1}$; $\boldsymbol{\lambda}^{\text{new}} \gets \boldsymbol{\lambda}_0$ \
\While{Adaptive Ridge ($\text{ItTol1}, \text{Tol1}$) does not converge}
    \State $\boldsymbol{\alpha} \gets \boldsymbol{\alpha}^{\text{new}}$
    \State $\boldsymbol{\omega} \gets \boldsymbol{\omega}^{\text{new}}$
    \While{Fellner-Schall ($\text{ItTol2}, \text{Tol2}$) does not converge}
    	\State $\tilde{\boldsymbol{\alpha}} \gets \boldsymbol{\alpha}^{\text{new}}$
        \State $\boldsymbol{\lambda} \gets \boldsymbol{\lambda}^{\text{new}}$
        \State $\boldsymbol{\alpha}^{\text{new}} \gets \left(\mathbf{B}'\mathbf{B}+\mathbf{P}^{\boldsymbol{\lambda}\boldsymbol{\omega}}_{\text{I}}\right)^{-1}\mathbf{B}'\mathbf{y}$ 
        \State $\boldsymbol{\lambda}^{\text{new}} \gets$ Fellner-Schall update ($\boldsymbol{\lambda},\boldsymbol{\alpha}^{\text{new}}$) as in \eqref{eq: Fellner-Schall for AKSSAM}
    \EndWhile
    \State $\boldsymbol{\omega}^{\text{new}} \gets$ update with $\boldsymbol{\alpha}^{\text{new}}$ via $\omega^{\text{new}}_{rl} \gets \left(\left(\Delta^{q_r+1} {\alpha}^{\text{new}}_{r\,q_r+1+l}\right)^2 + \varepsilon^2\right)^{-1},\quad
\left.\begin{aligned}
&l = 1, \ldots, k_r-2q_r-2 \\
&r = 1, \ldots, p
\end{aligned}\right.$
\EndWhile    
\State {$
\mathbf{t}^{\text{sel}}_r \gets  \left\{t_{q_r+1+l}\in\mathbf{t}^*_r\,|\, \omega^{\text{new}}_{rl}\left(\Delta^{q_r+1}\alpha^{\text{new}}_{r\,q_r+1+l}\right)^2 \approx 1\right\}_{l = 1}^{k_r-2q_r-2},\quad r = 1,\ldots,p 
$}
\State $\hat{\mathbf{t}}_r \gets \text{append external and boundary knots of $\mathbf{t}_r$ to $\mathbf{t}^{\text{sel}}_r$, $r=1,\ldots,p$}$
\State ${\boldsymbol\alpha}^{\text{sel}} \gets \left[(\mathbf{B}^{\text{sel}})'(\mathbf{B}^{\text{sel}}) + \mathbf{P}_{\text{I}}\right]^{-1}(\mathbf{B}^{\text{sel}})'\mathbf{y}$, where $\mathbf{B}^{\text{sel}}$ is the design matrix defined over $\hat{\mathbf{t}}_1,\ldots,\hat{\mathbf{t}}_p$.
\end{algorithmic}
\end{algorithm}

\subsubsection{AKSSAM for GAMs}\label{subsubsec:3.4.2}

The AM setting presented in Section~\ref{subsec:3.2} can be extended to GAMs \citep{tibshirani1996regression}, in which the response is allowed to follow any distribution belonging to the exponential family. AKSSAM can be as well extended from its additive formulation in Section~\ref{subsubsec:3.4.1} to the GAM setting through the P-IRLS algorithm and accordingly modifying the Fellner--Schall update.

Let $X_1,\ldots,X_p$ be $p$ continuous covariates and $Y$ the response following an exponential-family distribution with associated link function $g$, consider the GAM
\begin{equation}\label{eq:Generalized Additive model}
g(\mu_i) = \eta_i = \alpha_0 + f_1({x_{1i}}) + \ldots + f_p(x_{pi}),\qquad i=1,\ldots,n,
\end{equation}
where $f_1,\ldots, f_p$ are unknown smooth functions, $\alpha_0$ the intercept, $\{(x_{1i},\ldots,x_{pi},y_i)\}_{i=1}^n$ the set of observations, with $\mu_i$ and $\eta_i$ being the unknown expected value of $Y$ and linear predictor for the $i$-th observation, respectively.


As in \eqref{eq:Multivariate Smoothing}, each smooth function $f_r$ is represented using a B-spline basis expansion, with identifiability enforced as described in Section~\ref{subsec:3.2}. Since the estimation problem generally admits no closed-form solution, we adopt a penalized likelihood approach. As discussed in Section~\ref{subsec:2.2}, the resulting optimization problem can be solved using the P-IRLS algorithm.

Considering the identifiability penalty $\mathbf{P}_\text{I}$, the estimation problem in GAMs results in
\begin{equation}\label{eq: GAM Estimation}
\hat{\boldsymbol{\alpha}} \in \arg\max_{\boldsymbol{\alpha}\in\mathbb{R}^m} \ell(\boldsymbol \alpha) - \frac{1}{2}\boldsymbol{\alpha}'\mathbf{P}_{\text{I}}\boldsymbol{\alpha},
\end{equation}
where $\ell$ is the log-likelihood associated with the distribution of the response.

A-splines can be accommodated to GAMs by dealing with the maximum likelihood problem
\begin{equation}\label{eq: GAM A-splines}
 \arg\max_{\boldsymbol{\alpha}\in\mathbb{R}^m} \ell(\boldsymbol \alpha) - \frac{1}{2}\boldsymbol{\alpha}'\mathbf{P}^{\boldsymbol{\lambda}\boldsymbol{\omega}}_{\text{I}}\boldsymbol{\alpha},
\end{equation}
where $\mathbf{P}^{\boldsymbol{\lambda}\boldsymbol{\omega}}_{\text{I}}$ is given by \eqref{eq: Additive A-splines ii}. 

In this case, AKSSAM remains valid taking into account that the update of the regression coefficients has to be made iteratively through P-IRLS and 
that the Fellner--Schall update for a fixed-weight problem induced by \eqref{eq: GAM A-splines} is given by
\begin{equation}\label{eq: GAM Fellner-Schall for AKSSAM}
\lambda^*_r = \phi\frac{\text{tr}\left[\left({\mathbf{P}^{\boldsymbol{\lambda}\boldsymbol{\omega}}} \right)^-\mathbb{P}^{\boldsymbol{\omega}_r}\right]-\text{tr}\left[\left(\mathbf{B}'\Omega\mathbf{B} + \mathbf{P}^{\boldsymbol{\lambda}\boldsymbol{\omega}}_{\text{I}}\right)^{-1}\mathbb{P}^{\boldsymbol{\omega}_r}\right]}{\boldsymbol{\hat{\alpha}_{\boldsymbol{\lambda}}}'\mathbb{P}^{\boldsymbol{\omega}_r} \boldsymbol{\hat{\alpha}_{\boldsymbol{\lambda}}}} \lambda_r,\quad r=1,\ldots,p,
\end{equation}
where $\hat{\boldsymbol{\alpha}}_{\boldsymbol{\lambda}}$ are the penalized maximum-likelihood estimated parameters in \eqref{eq: GAM A-splines} for $\boldsymbol{\lambda}$, $\Omega$ is the weights matrix at convergence of P-IRLS used for estimating the coefficients and $\phi$ the scale parameter associated with the response's distribution. 

The justification of the validity of \eqref{eq: GAM Fellner-Schall for AKSSAM} is analogous to the argument exposed in Theorem~\ref{thm:AKSSAM THM} but noting that $\Omega$ cannot have null diagonal elements.

Therefore, the GAM version of AKSSAM is analogous to Algorithm~\ref{Alg: AKSSAM}, but requires changing the penalty parameter update in line~9 by \eqref{eq: GAM Fellner-Schall for AKSSAM} and the use of P-IRLS both at lines~8 and 15 for proper coefficient estimation.
\section{Computational experiments}
\label{sec:4}

In this section, we evaluate the performance of AKSSAM in a variety of additive (AM) and generalized additive model (GAM) settings. The experiments comprise two simulation studies, one with a Gaussian response and one with a Poisson response, together with two real-data applications: the \texttt{electric\_load} dataset (available in the R package \texttt{opera} \cite{R-opera}), involving a Gaussian response, and the \texttt{PimaIndiansDiabetes} dataset (available in the R package \texttt{mlbench} \cite{mlbench}), involving a binomial response. We compare AKSSAM with P-splines, the Particle Swarm Optimization (PSO) approach proposed by \citet{thielmann2025enhancing}, and the Geometrically Designed variable knot Spline regression algorithm (GeDS) introduced by \citet{Dimitrova2023}. Since our comparison is based on publicly available software, each method is evaluated only in those scenarios for which an implementation is available.

P-splines are fitted using the \texttt{gam} function from the R package \texttt{mgcv} \citep{mgcvR}, with smoothing parameters estimated by the Fellner--Schall (REML) algorithm. The PSO approach is implemented through the \texttt{fit\_gam\_pso} function available at \url{https://github.com/AnFreTh/OKPSPS}, using its default settings. Finally, GeDS is fitted using the \texttt{NGeDSgam} function from the R package \texttt{GeDS} \citep{GeDS}, which provides a GAM implementation of the method. Table~\ref{tab:SummaryChpt4} summarizes the data scenarios considered and the methods evaluated in each case.



It is worth noting that both the PSO and GeDS algorithms require additional tuning parameters to achieve their best performance. To ensure a fair comparison between automatic knot-selection procedures, we use the default values of these parameters throughout all experiments. Furthermore, the estimated smooth functions produced by GeDS are centered after model fitting, as the implementation does not enforce the usual identifiability constraints. This post-processing step makes the estimated component functions directly comparable with those obtained by the other methods.


Finally, all knot-selection procedures are initialized from the same deliberately overparameterized set of knots, dividing the observed covariates' support in $k_r-2q_r-1=40$ uniform-length intervals, i.e., since $q_r=3$, $k_r=47$ knots, for all $r=1,\ldots,p$. Starting from a large initial basis reflects the practical situation in which no prior information is available about the appropriate number of knots, allowing each method to determine the final model complexity through its own knot-selection mechanism. Although \citet{thielmann2025enhancing} considered considerably smaller initial knot sets (typically around 15 knots), we use $k=47$ for all methods to ensure a fair comparison under a common initialization and to evaluate their ability to discard unnecessary knots.


\begin{table}[ht]
\centering
\begin{tabular}{r|l|cccc}
\toprule
 \multicolumn{1}{c}{Setting} & Distribution & \textbf{AKSSAM} & \textbf{P-Splines} & \textbf{GeDS} & \textbf{PSO} \\
\midrule
{Simulation 1}  & Gaussian & $\checkmark$ & $\checkmark$ & $\checkmark$ & $\checkmark$ \\
{Simulation 2}  & Poisson & $\checkmark$ & $\checkmark$ & $\checkmark$ &  \\
{\texttt{electric\_load}} & Gaussian & $\checkmark$ & $\checkmark$ & $\checkmark$ & $\checkmark$ \\
{\texttt{PimaIndians}} & Binomial & $\checkmark$ & $\checkmark$ & $\checkmark$ &  \\
\bottomrule
\end{tabular}
\caption{Summary of computational experiments and {the algorithms evaluated in each scenario.}}
\label{tab:SummaryChpt4}
\end{table}

An implementation of AKSSAM in the R programming language, {together with the code required to reproduce all the experiments presented in this section, is publicly} available at \url{https://github.com/n-carrizo/AKSSAM}. {All experiments have been conducted} using R version 4.5.1 on a server equipped with an AMD Ryzen Threadripper 3990X processor (64 cores, 128 threads, 2.2–4.38 GHz), 256~GiB of RAM, and running Debian GNU/Linux 12 (bookworm).

The performance of the different methods is evaluated using metrics appropriate for each response type. Specifically, the mean squared error (MSE) is computed with respect to the response variable for the Gaussian scenarios and with respect to the linear predictor for the Poisson scenarios, whereas the area under the ROC curve (AUC) is used for the binomial scenario. In addition, we assess the complexity and goodness-of-fit of the fitted models by reporting the total number of parameters, defined as the number of B-spline basis functions plus the intercept term (Size), the effective degrees of freedom (EDF), and the Bayesian Information Criterion (BIC). For AKSSAM and GeDS, the size of the B-spline basis, which reflects the sparsity of the fitted model, coincides with its effective degrees of freedom. Finally, the runtime of each method (in seconds) is also reported.


\begin{remark}
Before proceeding, we highlight several practical considerations regarding the implementation of AKSSAM, including potential sources of numerical instability.
\begin{itemize}
        \item Concerning convergence Algorithm~\ref{Alg: AKSSAM}, neither the adaptive ridge scheme, due to the non-convexity of the underlying $L_0$-penalized problem, nor its combination with the Fellner--Schall updates enjoys global convergence guarantees. In our experiments, the median number of iterations of the outer loop required for our algorithm to converge has been at most 18 throughout the different instances, and the selected knot configurations have been robust to the choice of the initial penalty parameters $\boldsymbol{\lambda}_0 = 10\cdot\mathbf{1}^p$ .

		\item 
        To improve the numerical stability of Algorithm~\ref{Alg: AKSSAM}, if the update produces $\lambda_r^*<0$ as a result of numerical instabilities or finite-precision arithmetic, we replace it with a small positive value, for example $\lambda_r^*=10^{-1}$. In our experience, this simple safeguard may require a few additional iterations but substantially reduces the occurrence of non-convergence.
		
		\item 
        Notice that, with $\epsilon=10^{-5}$ in the adaptive ridge procedure, if $\Delta^{q_r+1}\alpha_{r\,q_r+1+l}=0$, then $\omega_{rl}\approx10^{10},$ $r=1,\ldots,p,\, l=1,\ldots,k_r-2q_r-2$. Consequently, the diagonal matrix $W_r$ in \eqref{eq: Additive A-splines ii} may contain entries that differ by several orders of magnitude. While this has not caused numerical issues in the Gaussian setting, response distributions requiring the P-IRLS algorithm may lead to ill-conditioning of the matrix $\mathbf{B}^{\prime}\Omega\mathbf{B}+\mathbf{P}_{\text{I}}^{\boldsymbol{\lambda}\boldsymbol{\omega}}$, owing to the combined effect of the large entries in $W_r$ and the P-IRLS weight matrix $\Omega$. This phenomenon is likely a consequence of finite-precision arithmetic. In such cases, we have found that imposing an upper bound on the updated penalty parameters effectively resolves the issue. We recommend using $10^{4}$ as this upper bound, since the converged penalty parameters rarely exceed this value.
\end{itemize}\end{remark}

\subsection{Simulations \label{subsec: 4.1}}

{
The following settings are used in all simulation studies. For every covariate ($r=1,\ldots,p$), a cubic B-spline basis ($q_r=3$) with an initial set of $k_r=47$ equally-spaced knots is considered and  AKSSAM is initialized with a penalty parameters $\lambda_r=10.$  Regarding tolerances of AKSSAM, we use $\text{Tol1}=10^{-5}$, $\text{Tol2}=10^{-5}$, $\text{Tol3}=10^{-5}$, $\text{ItTol1}=100$, $\text{ItTol2}=100$, and $\text{ItTol3}=100$. These values have been selected after extensive preliminary experimentation and have been found to provide robust performance across a wide range of scenarios. For P-splines, PSO, and GeDS, the same initial knot configuration is employed, while all remaining algorithm-specific parameters are kept at their default values.
}


{
Additional results are provided in Appendix~\ref{app: Sim Setting}, including boxplots of the evaluation metrics across simulation replicates and the mean estimated function for each covariate in both simulation scenarios.
}


\subsubsection{Simulation 1: Gaussian}

The first simulation study considers the additive model defined in \eqref{eq:Additive model}, with $p=6$ covariates independently generated from the uniform distribution, $X_r\sim\mathcal{U}[0,1]$, for $r=1,\ldots,6$, and Gaussian errors $\varepsilon\sim\mathcal{N}(0,\sigma_{\text{error}}^2)$. The smooth component functions $f_r$ are listed in Table~\ref{tab:functions Gaussian} and illustrated in Figure~\ref{fig: Covariates_Sim2}. Each function is centered to have zero mean.


\begin{table}[H]
\renewcommand{\arraystretch}{1.6}
\centering
\begin{tabular}{c l} \hline
\textbf{Function} & \textbf{Expression} \\
\hline
$f_1(x)$ & $\cos\left(2\pi(2x - 2.5)\right)$ \\
$f_2(x)$ & $3e^{-70(x - 0.3)^2} + 3e^{-10(x - 0.7)^2} + \frac{x}{10}$ \\
$f_3(x)$ & $x^3 - 2\sin\left(\pi x^3\right)$ \\
$f_4(x)$ & $\frac{1}{3} \left(x + 10e^{-20(x - 0.5)^2} \right)$ \\
$f_5(x)$ & $\cos\left(6(2x - 3)\right) - 2.5(x - 1)$ \\
$f_6(x)$ & $5\sqrt{x(1 - x)} \cdot \sin\left( 2\pi\frac{\left(1 + 2^{-3/5}\right)}{x + 2^{-3/5}} \right)$ \\ \hline

\end{tabular}
\caption{Smooth functions for each covariate in {Simulation} 1.}
\label{tab:functions Gaussian}
\end{table}

\begin{figure}[htbp]
	\centering
	\includegraphics[width=0.9\textwidth]{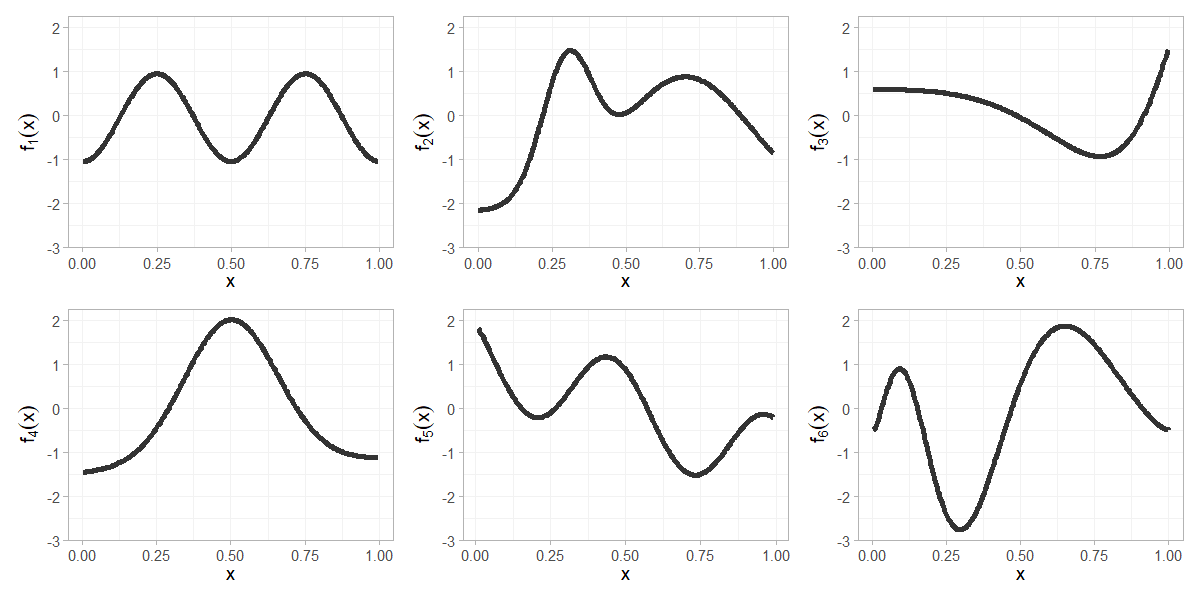}
	\caption{$f_r(x)$ in Table~\ref{tab:functions Gaussian}, $r=1,\ldots,6$.}
	\label{fig: Covariates_Sim2}
\end{figure}

The simulation study is conducted under different combinations of sample size and noise level. Specifically, we consider sample sizes $n\in\{300,500,800\}$ and error variances $\sigma_{\text{error}}^2$ chosen to yield signal-to-noise ratios (SNRs) $\in \{1,2,4\}$. Each of the resulting nine simulation settings is replicated 100 times, and all performance metrics are computed on the training data.


{
Table~\ref{tab:MultivarGaussianSimulation} reports the mean value of each performance metric over the 100 replicates, with the corresponding standard deviations shown in parentheses. The rows correspond to the different simulation settings, identified by the combination of sample size (first column) and SNR (second column). For each setting, the best average value of each metric is highlighted.
}


\begin{table}[H]
	\centering
	\scriptsize
	\setlength{\tabcolsep}{3pt}
	\renewcommand{\arraystretch}{0.9}

	\begin{tabular}{lr|rrrr|rrrrr}
		\toprule
		\multicolumn{2}{c}{} &
		\multicolumn{4}{c}{\textbf{AKSSAM}} &
		\multicolumn{5}{c}{\textbf{P-Splines}} \\
		\midrule
		$n$ & SNR & MSE & BIC & EDF/Size & Time &
		MSE & BIC & EDF & Size & Time \\
		\midrule
		300 & 1
			& 0.94 (0.21) & 1566 (28) & 36.91 (3.42) &  11.76 (8.63)
			& \cellcolor{gray!50} 0.85 (0.24) & \cellcolor{gray!50} 1538 (23) & \cellcolor{gray!50} 31.67 (3.24) & 259 (0) & 1.14 (0.43) \\
		& 2
			& 0.47 (0.11) & 1363 (26) & \cellcolor{gray!50} 38.83 (2.26) &  7.22 (5.81)
			& \cellcolor{gray!50} 0.42 (0.09) & \cellcolor{gray!50} 1358 (25) &  39.35 (2.03) & 259 (0) &  0.84 (0.31) \\
		& 4
			& 0.24 (0.05) & \cellcolor{gray!50} 1160 (26) & \cellcolor{gray!50} 40.11 (1.53) & 6.01 (6.51)
			& \cellcolor{gray!50} 0.22 (0.04) & 1185 (25) & 46.22 (1.87) & 259 (0) &  0.64 (0.08) \\
		\midrule
		500 & 1
			& 0.58 (0.12) &2517 (34) & 38.37 (2.05) & 9.07 (5.95)
			& \cellcolor{gray!50} 0.49 (0.10) & \cellcolor{gray!50} 2509 (32) & 37.81 (2.15) & 259 (0) &  1.00 (0.12) \\
		& 2
			& 0.29 (0.06) & \cellcolor{gray!50} 2176 (34) & 39.84 (1.66) & 6.25 (4.33)
			& \cellcolor{gray!50} 0.27 (0.05) & 2197 (31) & 44.82 (1.71) & 259 (0) &  0.84 (0.12) \\
		& 4
			& 0.16 (0.03) & \cellcolor{gray!50} 1835 (34) & \cellcolor{gray!50} 41.26 (1.62) &  4.09 (0.96)
			& \cellcolor{gray!50} 0.14 (0.03) & 1883 (36) & 52.67 (1.58) & 259 (0) &  0.70 (0.08) \\
		\midrule
		800 & 1
			& 0.36 (0.08) & \cellcolor{gray!50} 3933 (49) & 39.78 (2.11) &  7.52 (6.29)
			& \cellcolor{gray!50} 0.32 (0.06) & 3947 (40) & 43.28 (1.84) & 259 (0) &  1.20 (0.17) \\
		& 2
			& 0.19 (0.04) & \cellcolor{gray!50} 3384 (47) &  41.03 (1.65) & 5.12 (1.85)
			& \cellcolor{gray!50} 0.17 (0.03) & 3427 (40) & 50.94 (1.52) & 259 (0) &  0.96 (0.15) \\
		& 4
			& 0.10 (0.02) & \cellcolor{gray!50} 2837 (48) &  42.62 (1.53) & 3.89 (1.44)
			& \cellcolor{gray!50} 0.10 (0.01) & 2921 (43) & 58.42 (1.36) & 259 (0) & \cellcolor{gray!50}   0.84 (0.09) \\
		\bottomrule
	\end{tabular}

	\vspace{1em}

	\begin{tabular}{lr|rrrr|rrrrr}
		\toprule
		\multicolumn{2}{c}{} &
		\multicolumn{4}{c}{\textbf{GeDS}} &
		\multicolumn{5}{c}{\textbf{PSO}} \\
		\midrule
		$n$ & SNR & MSE & BIC & EDF/Size & Time &
		MSE & BIC & EDF & Size & Time \\
		\midrule
		300 & 1
			& 1.34 (0.33) & 1604 (26) & 37.23 (4.28) & \cellcolor{gray!50} 0.40 (0.14)
			& 0.99 (0.37) & 1577 (42) & 44.59 (11.61) & 193.30 (40.71) & 590 (21) \\
		& 2
			& 0.88 (0.29) & 1425 (28) & 39.97 (4.84) & \cellcolor{gray!50} 0.45 (0.12)
			& 0.52 (0.14) & 1387 (38) & 49.07 (10.04) & 222.63 (30.47) & 591 (16) \\
		& 4
			& 0.54 (0.25) & 1249 (44) & 44.15 (4.85) & \cellcolor{gray!50} 0.53 (0.15)
			& 0.26 (0.07) & 1201 (37) & 52.34 (9.46) & 231.84 (23.43) & 610 (15) \\
		\midrule
		500 & 1
			& 1.17 (0.32) & 2571 (38) & \cellcolor{gray!50} 35.03 (2.90) & \cellcolor{gray!50} 0.54 (0.18)
			& 0.55 (0.13) & 2539 (44) & 45.66 (8.66) & 187.23 (34.80) & 636 (22) \\
		& 2
			& 0.83 (0.35) &  2269 (54) & \cellcolor{gray!50} 37.98 (3.75) & \cellcolor{gray!50} 0.61 (0.16)
			& 0.30 (0.07) & 2216 (38) & 50.29 (6.55) & 212.78 (29.89) & 662 (22) \\
		& 4
			& 0.50 (0.25) & 1961 (67) & 41.78 (4.65) & \cellcolor{gray!50} 0.69 (0.18)
			& 0.16 (0.04) & 1892 (47) & 55.71 (7.38) & 228.43 (25.33) & 659 (14) \\
		\midrule
		800 & 1
			& 1.09 (0.33) & 4015 (60) & \cellcolor{gray!50} 34.02 (2.45) & \cellcolor{gray!50} 0.76 (0.23)
			& 0.38 (0.10) & 3978 (52) & 50.78 (8.30) & 186.56 (34.98) & 742 (29) \\
		& 2
			& 0.81 (0.32) & 3537 (77) & \cellcolor{gray!50} 36.90 (3.50) & \cellcolor{gray!50} 0.95 (0.28)
			& 0.21 (0.06) & 3449 (58) & 56.61 (8.95) & 207.48 (36.66) & 737 (20) \\
		& 4
			& 0.51 (0.28) & 3040 (108) & \cellcolor{gray!50} 41.52 (4.49) &  1.13 (0.31)
			& 0.11 (0.02) & 2923 (53) & 59.63 (6.95) & 225.76 (28.46) & 772 (23) \\
		\bottomrule
	\end{tabular}

	\caption{{    Mean performance metrics (standard deviations in parentheses) across 100 replicates for Simulation 1.}}
	\label{tab:MultivarGaussianSimulation}
\end{table}

{
Table~\ref{tab:MultivarGaussianSimulation} shows that AKSSAM achieves the lowest average BIC in six of the nine simulation settings, namely the most informative ones ($n=300$ with $\mathrm{SNR}=4$, $n=500$ with $\mathrm{SNR}\in\{2,4\}$, and all settings with $n=800$). In the remaining, less informative scenarios ($n=300$ with $\mathrm{SNR}\in\{1,2\}$ and $n=500$ with $\mathrm{SNR}=1$), its performance is comparable to that of P-splines. In the most informative settings, AKSSAM produces substantially sparser models than P-splines while maintaining a similar predictive accuracy, resulting in consistently lower BIC values. By contrast, in the most challenging scenario ($n=300$, $\mathrm{SNR}=1$), P-splines achieves both higher predictive accuracy and simpler models, suggesting that it is better suited to settings with limited and highly noisy training data. In terms of computational time, P-splines is consistently faster than AKSSAM, although the additional runtime required by AKSSAM remains moderate.
}


{
GeDS is consistently the fastest method. Moreover, in the most informative scenarios, it produces the sparsest models in terms of EDF. However, this reduction in model complexity comes at the expense of predictive performance, yielding the highest MSE values and larger BIC values than both AKSSAM and P-splines across all simulation settings. Although GeDS typically retains fewer knots than AKSSAM, the resulting loss in predictive accuracy outweighs the gain in sparsity.
}


{
PSO exhibits the weakest overall performance. It produces the most complex models and requires computational times that are orders of magnitude larger than those of the competing methods, without achieving any improvement in predictive accuracy. Starting from an initial basis of 259 B-spline basis functions, PSO typically retains around 200, in contrast to the approximately 40 basis functions selected by AKSSAM and GeDS.
}


\subsubsection{{Simulation} 2: Poisson}

{ The second simulation study considers the generalized additive model defined in \eqref{eq:Generalized Additive model}, with a Poisson response and $p=3$ covariates independently generated from the uniform distribution, $X_r\sim\mathcal{U}[0,1]$, for $r=1,2,3$. The smooth component functions $f_r$ are listed in Table~\ref{tab:functions Poisson} and illustrated in Figure~\ref{fig: Covariates_Sim3}. Each function is centered to have zero mean.}


\begin{table}[H]
\renewcommand{\arraystretch}{1.5}
\centering
\begin{tabular}{c l}
\hline
\textbf{Function} & \textbf{Expression} \\
\hline
$f_1(x)$ & $\frac{1}{2} \sin(4\pi x) \log(20x + 1)$ \\
$f_2(x)$ & $x^3 + \sin\left(2\pi(x^7 - 2x^3)\right)$ \\
$f_3(x)$ & $\frac{1}{2} \left(\sin(6\pi x) + 3x - 5x^2 + \cos(3\pi x)\right)$ \\
\hline
\end{tabular}
\caption{Smooth functions for each covariate in { Simulation} 2.}
\label{tab:functions Poisson}
\end{table}

\begin{figure}[htbp]
	\centering
	\includegraphics[width=0.9\textwidth]{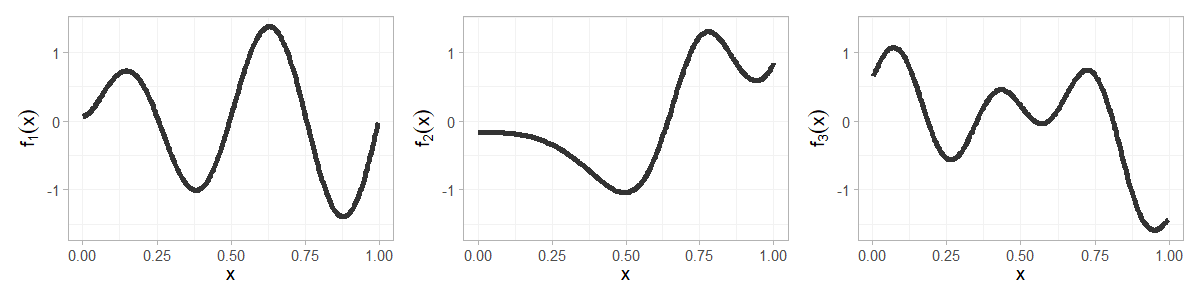}
	\caption{$f_r(x)$ in Table~\ref{tab:functions Poisson}, $r=1,2,3$.}
	\label{fig: Covariates_Sim3}
\end{figure}

{
The simulation study is conducted for three different sample sizes, $n\in\{300,500,800\}$. Each simulation setting is replicated 100 times, and all performance metrics are computed on the training data.

Table~\ref{tab:PoissonSimulation} reports the mean value of each performance metric over the 100 simulation replicates, with the corresponding standard deviations shown in parentheses. The rows correspond to the three simulation settings, identified by the sample size $n$.
}


\begin{table}[H]
	\centering
	\scriptsize
	\setlength{\tabcolsep}{3pt}
	\renewcommand{\arraystretch}{0.9}

	\begin{tabular}{l|rrrr|rrrrr}
		\toprule
		\multicolumn{1}{c}{} &
		\multicolumn{4}{c}{\textbf{AKSSAM}} &
		\multicolumn{5}{c}{\textbf{P-Splines}} \\
		\midrule
		$n$ & MSE{$\cdot 10^{-2}$} & BIC & EDF/Size & Time &
		MSE{$\cdot 10^{-2}$} & BIC & EDF & Size & Time \\
		\midrule
		300 & 8.80 (4.80) & \cellcolor{gray!50} 916 (25) &  23.49 (1.24) & 0.84 (0.19)
			& \cellcolor{gray!50} 7.90 (2.90) & 945 (25) & 29.66 (0.93) & 130 (0) & \cellcolor{gray!50} 0.27 (0.03) \\
		500 & 4.90 (2.10) & \cellcolor{gray!50} 1474 (33) &  24.24 (1.33) & 0.77 (0.10)
			& \cellcolor{gray!50} 4.60 (1.60) & 1520 (31) & 33.35 (0.88) & 130 (0) & \cellcolor{gray!50} 0.34 (0.03) \\
		800 & 3.10 (1.10) & \cellcolor{gray!50} 2294 (38) &  25.09 (1.64) & 0.77 (0.11)
			& \cellcolor{gray!50} 3.00 (1.10) & 2360 (38) & 36.97 (0.67) & 130 (0) & \cellcolor{gray!50} 0.42 (0.04) \\
		\bottomrule
	\end{tabular}

	\vspace{1em}

	\begin{tabular}{l|rrrr}
		\toprule
		\multicolumn{1}{c}{} &
		\multicolumn{4}{c}{\textbf{GeDS}} \\
		\midrule
		$n$ & MSE{$\cdot 10^{-2}$} & BIC & EDF/Size & Time \\
		\midrule
		300 &  42.90 (26.30) & 1061 (107) & \cellcolor{gray!50} 19.74 (2.91) &  0.31 (0.09) \\
		500 &  36.90 (26.40) & 1709 (196) & \cellcolor{gray!50} 20.31 (3.15) &  0.49 (0.12) \\
		800 &  37.20 (26.00) & 2694 (274) & \cellcolor{gray!50} 20.12 (2.93) &  0.79 (0.19) \\
		\bottomrule
	\end{tabular}

	\caption{{    Mean performance metrics (standard deviations in parentheses) across 100 replicates for Simulation 2.}}
	\label{tab:PoissonSimulation}
\end{table}

{
Table~\ref{tab:PoissonSimulation} shows that AKSSAM achieves predictive performance comparable to that of P-splines across all simulation settings, although with slightly larger prediction errors. In return, AKSSAM consistently produces simpler models, reducing the effective degrees of freedom by approximately $10$ for $n\in\{500,800\}$. P-splines is also consistently faster than AKSSAM, although both methods require less than one second on average. Overall, these results indicate that AKSSAM attains a favorable trade-off between predictive accuracy and model complexity, yielding substantially sparser models while maintaining a level of accuracy close to that of P-splines.

GeDS, by contrast, performs considerably worse than both AKSSAM and P-splines. Although it is consistently the fastest method and produces the sparsest models in terms of EDF, its prediction errors are approximately one order of magnitude larger, leading to the least favorable BIC values across all simulation settings. As in the Gaussian simulation study, these results suggest that the aggressive reduction in model complexity achieved by GeDS comes at the expense of predictive performance.
}



\subsection{Real Data \label{subsec: 4.2}}

{
The following settings are used in all real-data experiments. For every covariate ($r=1,\ldots,p$), a cubic B-spline basis ($q=3$) with an initial set of $k=47$ equally-spaced knots is considered and AKSSAM is initialized with a penalty parameter $\lambda_r=10.$ Regarding tolerances of AKSSAM, we use $\text{Tol1}=10^{-5}$, $\text{Tol2}=10^{-5}$, $\text{Tol3}=10^{-5}$, $\text{ItTol1}=100$, $\text{ItTol2}=100$, and $\text{ItTol3}=100$. For P-splines, PSO, and GeDS, the same initial knot configuration is employed, while all remaining algorithm-specific parameters are kept at their default values. Predictive performance is assessed using 5-fold cross-validation.
}


Additional results are provided in  Appendix~\ref{app: Real Setting}, including  boxplots of the evaluation metrics across the folds.

\subsubsection{\texttt{electric\_load} dataset}

{To evaluate AKSSAM on a real-world additive modeling problem, we consider the \texttt{electric\_load} dataset, available in the \texttt{opera} package for \texttt{R} \citep{R-opera}. The dataset contains $n=731$ observations collected in France between 1996 and 2009, including electricity demand, temperature measurements, and industrial production indices. The variables used in our analysis are the following:}

\begin{itemize}
        \item {\texttt{Load1}}: Lagged electrical load.
        \item \texttt{NumWeek}: Value between 0 and 1 representing the fraction of the year. 
        \item \texttt{Temp}: Temperature in celsius (ºC).
        \item \texttt{IPI}: Industrial production index.
\end{itemize}

{ 
The objective is to model \texttt{Load1} using the AM in \eqref{eq:Additive model}, with \texttt{NumWeek}, \texttt{Temp}, and \texttt{IPI} as covariates, while automatically selecting the knots defining the B-spline basis for each smooth component. Both the response variable and the covariates are standardized prior to model fitting.}


{
Table~\ref{tab:electric.load} reports the performance metrics obtained in each of the five cross-validation folds, with each row corresponding to one fold. For each fold, the best value of each metric across methods is highlighted. The mean and standard deviation of each metric over the five folds are also included.
}


\begin{table}[H]
	\centering
	\scriptsize
	\setlength{\tabcolsep}{3pt}
	\renewcommand{\arraystretch}{0.9}
	\begin{tabular}{l|rrrr|rrrrr}
		\toprule
		\multicolumn{1}{c}{} &
		\multicolumn{4}{c}{\textbf{AKSSAM}} &
		\multicolumn{5}{c}{\textbf{P-Splines}} \\
		\midrule
		Fold & MSE & BIC & EDF/Size & Time &
		MSE & BIC & EDF & Size & Time \\
		\midrule
		1 & \cellcolor{gray!50} 0.23 & \cellcolor{gray!50} 792.82 & \cellcolor{gray!50} 19.00 &  1.50
		  & \cellcolor{gray!50} 0.23 & 802.13 & 22.60 & 130 & 0.28 \\
		2 & 0.25 & \cellcolor{gray!50} 796.02 & \cellcolor{gray!50} 20.00 &  0.66
		  & \cellcolor{gray!50}0.23 &  811.21 & 24.87 & 130 & 0.27 \\
		3 & 0.22 & \cellcolor{gray!50} 793.10 & \cellcolor{gray!50} 18.00 & 1.32
		  & \cellcolor{gray!50} 0.21 &  808.12 & 21.90 & 130 & 0.35 \\
		4 & \cellcolor{gray!50} 0.16 & \cellcolor{gray!50} 837.40 & \cellcolor{gray!50} 19.00 &  0.85
		  & \cellcolor{gray!50} 0.16 &  848.17 & 23.39 & 130 & 0.39 \\
		5 & \cellcolor{gray!50} 0.20 & 826.22 & \cellcolor{gray!50} 20.00 &  0.61
		  & \cellcolor{gray!50} 0.20 & \cellcolor{gray!50} 823.20 & 22.73 & 130 & \cellcolor{gray!50} 0.23 \\
        \midrule 
         Mean (Sd)& 0.21 (0.03) & \cellcolor{gray!50} 809 (21) & \cellcolor{gray!50} 19.2 (0.8) & 0.99 (0.02) & 0.21 (0.03) & 819 (18) & 23.1 (1.1) & 130 (0) & 0.30 (0.06) \\
		\bottomrule
	\end{tabular}
	\vspace{1em}
    
	\begin{tabular}{l|rrrr|rrrrr}
		\toprule
		\multicolumn{1}{c}{} &
		\multicolumn{4}{c}{\textbf{GeDS}} &
		\multicolumn{5}{c}{\textbf{PSO}} \\
		\midrule
		Fold & MSE & BIC & EDF/Size & Time &
		MSE & BIC & EDF & Size & Time \\
		\midrule
		1 & \cellcolor{gray!50} 0.23 & 815.03 & 24.00 & \cellcolor{gray!50} 0.18
		  & \cellcolor{gray!50} 0.23 & 870.22 & 39.27 & 130 & 4.24 \\
		2 &\cellcolor{gray!50} 0.23 & 827.30 &  27.00 & \cellcolor{gray!50} 0.19
		  & 0.24 & 963.84 & 62.55 & 130 & 3.81 \\
		3 & \cellcolor{gray!50} 0.21 & 830.62 & 25.00 & \cellcolor{gray!50} 0.17
		  & \cellcolor{gray!50} 0.21 & 949.82 & 57.24 & 130 & 4.57 \\
		4 & 0.17 & 862.57 & 25.00 & \cellcolor{gray!50} 0.18
		  & 0.17 & 905.81 & 37.44 & 130 & 3.81 \\
		5 & \cellcolor{gray!50} 0.20 & 848.47 & 26.00 & \cellcolor{gray!50} 0.23
		  & 0.22 & 929.03 & 48.28 & 130 & 4.09 \\
        \midrule 
        Mean (Sd) & 0.21 (0.03) & 837 (19) & 25.4 (1.1) & \cellcolor{gray!50} 0.19 (0.02) & \cellcolor{gray!50} 0.21 (0.02) & 924 (37) & 49.0 (11) & 130 (0) & 4.10 (0.30) \\
		\bottomrule
	\end{tabular}
	\caption{Performance metrics and their average across 5 folds for \texttt{electric\_load}.}
	\label{tab:electric.load}
\end{table}

{
Table~\ref{tab:electric.load} shows that AKSSAM and P-splines achieve nearly identical predictive performance across all five cross-validation folds. However, AKSSAM consistently produces substantially simpler models, with an average EDF of approximately 19, compared with larger EDF values for P-splines. This favorable trade-off between predictive accuracy and model complexity results in AKSSAM achieving the lowest BIC in four out of the five folds, with P-splines consistently ranking second.

GeDS also produces sparser models than P-splines and PSO. Nevertheless, AKSSAM attains comparable predictive performance while using approximately 5 to 7 fewer basis functions on average, leading to more favorable BIC values. These results indicate that the knot-selection strategy adopted by AKSSAM is more effective at balancing model complexity and predictive accuracy.

PSO exhibits the weakest overall performance. It fails to substantially reduce the initial knot set, producing models with EDF values ranging from approximately 37 to 62, while offering no improvement in predictive accuracy. Consequently, the resulting smoothing is insufficient, leading to unnecessarily complex models. 

In terms of computational time, GeDS is consistently the fastest method, followed by P-splines, AKSSAM, and PSO. Although AKSSAM requires between two and five times the runtime of P-splines and GeDS, its computational cost remains modest, whereas PSO is between three times and up to an order of magnitude slower than the other methods.
}


{
Figure~\ref{fig: Covariates_electic.load} displays the mean estimated smooth functions for each covariate, together with the distribution of the selected knots along the $x-$axis. The shading intensity reflects the frequency with which each knot is selected across the cross-validation folds, with darker regions indicating more frequently selected knots.
}

\begin{figure}[htbp]
	\centering
	\includegraphics[width=\textwidth]{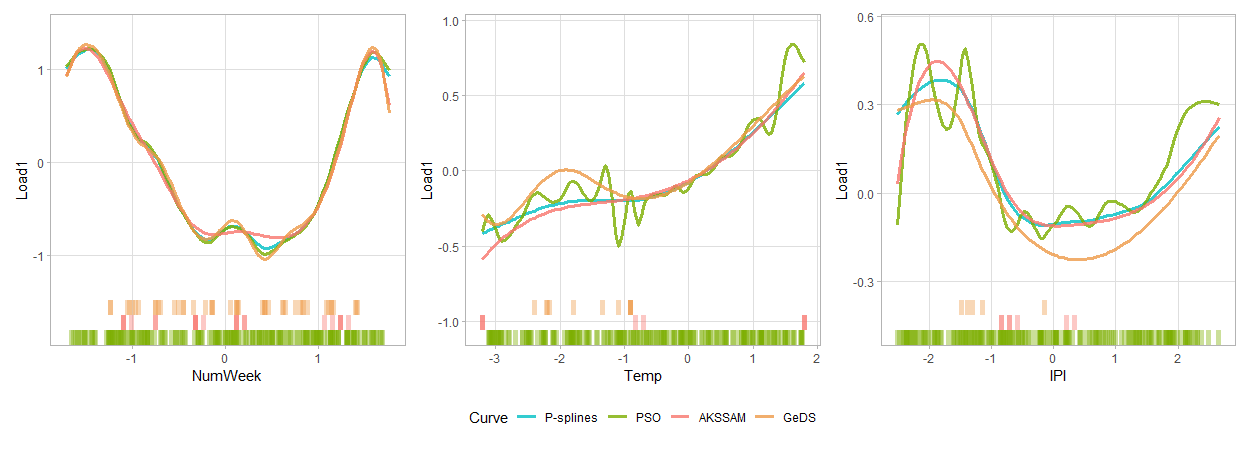}
	\caption{
    Mean estimated smooth functions for each covariate and distribution of the selected knots across the 5 folds in \texttt{electric\_load}.}
	\label{fig: Covariates_electic.load}
\end{figure}

{
Regarding the estimated smooth functions, PSO fails to adequately smooth the effects of \texttt{Temp} and \texttt{IPI}, resulting in noticeably rougher estimates than those obtained with the other methods. By contrast, P-splines, GeDS, and AKSSAM produce similar estimated functions for all covariates, although AKSSAM yields a slightly smoother estimate for \texttt{NumWeek}. The distribution of the selected knots obtained with AKSSAM is both sparse and consistent across the cross-validation folds, with higher selection frequencies concentrated in well-defined regions of the covariate domain. In contrast, PSO exhibits a much more dispersed selection pattern, with no clearly identifiable regions of frequent knot selection. GeDS also shows a less satisfactory knot configuration for \texttt{NumWeek}, both in terms of the number and the location of the selected knots, which may be attributable to the limitations of the backfitting procedure.
}


\subsubsection{\texttt{PimaIndians} dataset}

{
To evaluate AKSSAM on a real-world generalized additive modeling problem with a binomial response, we consider the \texttt{PimaIndiansDiabetes} dataset, available in the \texttt{mlbench} package for \texttt{R} \citep{mlbench}. The dataset contains $n=768$ observations, each corresponding to a patient and including several physiological measurements together with a binary indicator of diabetes. The variables used in our analysis are the following:
}

\begin{itemize}
        \item {\texttt{diabetes}}: Binary variable indicating the presence of diabetes on the patient.
        \item \texttt{mass}: Body mass index of the patient. 
        \item \texttt{age}: Age of the patient in years.
\end{itemize}

{The objective is to model \texttt{diabetes} using the GAM in \eqref{eq:Generalized Additive model}, with \texttt{mass} and \texttt{age} as covariates, while automatically selecting the knots defining the spline basis for each smooth component. Both covariates are standardized prior to model fitting.}


{
Table~\ref{tab:PimaIndians} reports the performance metrics obtained for each of the five cross-validation folds. The best value of each metric across methods is highlighted for every fold. The mean and standard deviation of each metric over the five folds are also reported.}

\begin{table}[H]
	\centering
	\scriptsize
	\setlength{\tabcolsep}{3pt}
	\renewcommand{\arraystretch}{0.9}
\resizebox{\textwidth}{!}{
	\begin{tabular}{l|rrrr|rrrrr|rrrr}
		\toprule
		\multicolumn{1}{c}{} &
		\multicolumn{4}{c}{\textbf{AKSSAM}} &
		\multicolumn{5}{c}{\textbf{P-Splines}} &
		\multicolumn{4}{c}{\textbf{GeDS}} \\
		\midrule
		Fold & AUC & BIC & EDF/Size & Time &
		AUC & BIC & EDF & Size & Time &
		AUC & BIC & EDF/Size & Time \\
		\midrule
		1 & \cellcolor{gray!50} 0.70 & 726.88 &  14.00 & 1.22
		  & \cellcolor{gray!50} 0.70 & \cellcolor{gray!50} 705.75 & \cellcolor{gray!50} 11.14 & 87 & 0.23
		  & \cellcolor{gray!50} 0.70 & 757.12 & 18.00 & \cellcolor{gray!50} 0.13 \\
		2 & \cellcolor{gray!50} 0.77 & 714.72 &  12.00 & 0.73
		  & \cellcolor{gray!50} 0.77 & \cellcolor{gray!50} 700.14 & \cellcolor{gray!50} 9.66 & 87 & 0.21
		  & 0.75 & 746.08 & 14.00 & \cellcolor{gray!50} 0.09 \\
		3 & 0.79 & 737.68 &  12.00 & 1.01
		  & \cellcolor{gray!50} 0.80 & \cellcolor{gray!50} 724.93 & \cellcolor{gray!50} 9.90 & 87 & 0.35
		  & 0.79 & 764.84 & 13.00 & \cellcolor{gray!50} 0.08 \\
		4 & \cellcolor{gray!50} 0.78 & 732.92 &  12.00 & 0.86
		  & \cellcolor{gray!50} 0.78 & \cellcolor{gray!50} 715.75 & \cellcolor{gray!50} 9.34 & 87 & 0.32
		  & 0.77 & 767.06 & 17.00 & \cellcolor{gray!50} 0.11 \\
		5 & 0.73 & 731.32 &  12.00 & 1.06
		  & \cellcolor{gray!50} 0.74 & \cellcolor{gray!50} 715.93 & \cellcolor{gray!50} 9.37 & 87 & 0.23
		  & 0.72 & 760.55 & 14.00 & \cellcolor{gray!50} 0.10 \\
        \midrule
         Mean (Sd) & 0.76 (0.04) & 729 (9) & 12.4 (0.9) & 0.97 (0.19) & \cellcolor{gray!50} 0.76 (0.03) & \cellcolor{gray!50} 712 (10) & \cellcolor{gray!50} 9.87 (0.74) & 87 (0) & 0.26 (0.06) & 0.75 (0.02) & 759 (8) & 15.2 (2.1) & \cellcolor{gray!50} 0.10 (0.02) \\
		\bottomrule
	\end{tabular}}
	\caption{Performance metrics and their average across 5 folds for \texttt{PimaIndians}.}
	\label{tab:PimaIndians}
\end{table}

{ 
Table~\ref{tab:PimaIndians} shows that P-splines consistently achieves the best predictive performance across all five cross-validation folds, yielding the lowest BIC values and the smallest EDF while attaining the highest AUC. AKSSAM closely matches its predictive accuracy in most folds, although with slightly larger EDF values (approximately three additional degrees of freedom) and runtimes roughly twice as large. However, AKSSAM produces substantially sparser spline representations, reducing the initial basis from 87 basis functions to only 12–14 while maintaining comparable predictive performance.

GeDS is consistently outperformed by AKSSAM. Although it is the fastest method, it produces more complex models than AKSSAM while achieving lower predictive accuracy, resulting in less favorable overall performance.
}


\begin{figure}[htbp]
	\centering
	\includegraphics[width=0.9\textwidth]{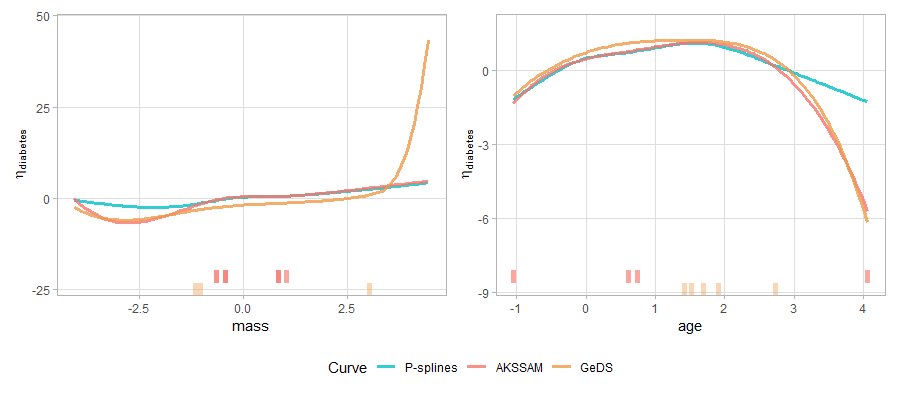}
	\caption{Mean estimated smooth functions for each covariate and distribution of the selected knots across the 5 folds in \texttt{PimaIndians}.}
	\label{fig: Covariates_PimaIndians}
\end{figure}

{ Figure~\ref{fig: Covariates_PimaIndians} displays the mean estimated smooth functions for each covariate, together with the distribution of the selected knots along the $x-$axis. The shading intensity reflects the frequency with which each knot is selected across the cross-validation folds, with darker regions indicating more frequently selected knots.}


{
Figure~\ref{fig: Covariates_PimaIndians} shows that P-splines produces smoother estimated functions than both AKSSAM and GeDS, particularly in the rightmost part of the effect of \texttt{age}. This behavior may be explained by the requirement that the knot sets selected by AKSSAM retain the 2q boundary knots, which imposes a lower bound on the achievable model complexity that is not present in P-splines. It is also noteworthy that AKSSAM and GeDS select different knot configurations for both covariates, suggesting that the optimal knot-selection problem may admit multiple solutions. Finally, AKSSAM exhibits a more consistent knot-selection pattern across the cross-validation folds than GeDS, whose estimated effect of \texttt{mass} deviates noticeably from those produced by the other methods.
}

\section{Conclusion}
\label{sec:5}


{
In this paper, we have addressed the problem of automatic knot selection in smooth generalized additive models by extending the recently proposed A-splines methodology \citep{goepp2025spline} to the additive setting. To avoid the computational burden of tuning multiple smoothing parameters through grid search, we proved that the conditions required by the Fellner--Schall algorithm \citep{wood2017generalized} hold in this context, enabling the simultaneous automatic estimation of the penalty parameters. This has resulted in the development of AKSSAM (Automatic Knot Selection in Smooth Additive Models), a fully automatic knot-selection algorithm for generalized additive models.
}


{
The proposed methodology has been evaluated for Gaussian, Poisson, and binomial responses through extensive simulation studies and real-data applications. Across them, AKSSAM has consistently achieved a favorable balance between predictive performance and model complexity. In informative settings, it has produced substantially sparser models than P-splines while maintaining comparable predictive accuracy, leading to improved BIC values. Compared with existing knot-selection methods, AKSSAM has consistently outperformed GeDS in terms of the sparsity–accuracy trade-off and dramatically reduced the number of retained B-spline basis functions relative to PSO, while requiring computational times that remained practical across all scenarios considered. Although P-splines remained preferable in the most challenging low-information settings, AKSSAM has demonstrated clear advantages whenever the available data were sufficiently informative. The experiments on real datasets have further confirmed that AKSSAM can drastically reduce the number of B-spline basis functions while preserving predictive performance, highlighting the effectiveness of pure knot selection as an alternative to traditional smoothing approaches.
}

{
The proposed methodology also opens several avenues for future research. From a methodological perspective, an immediate extension would be to broaden the class of supported response distributions beyond the exponential family. Other natural developments include extending AKSSAM to richer spline representations, such as thin-plate splines for bivariate smooth terms, which are widely used in geostatistics, or spline bases of higher dimensionality. It would also be interesting to combine automatic knot selection with additional model components, such as automatic variable selection \citep{kim2026simultaneous} or shape constraints \citep{NAVARROGARCIA,navarro2024mathematical}, yielding sparse, interpretable, and structurally constrained additive models. Another promising direction is the joint optimization of knot locations and roughness penalization, following the ideas of \citet{thielmann2025enhancing}.

Beyond methodological developments, the sparsity achieved by AKSSAM makes it particularly attractive for applications in optimization and machine learning. The substantial reduction in the number of basis functions observed throughout the computational experiments suggests that AKSSAM could provide compact surrogate models suitable for embedding objective functions or constraints into mathematical optimization problems, as in \citet{cuesta2025}. Other potential applications include automatic variable discretization through changepoint identification \citep{hollaway2024detection} and the determination of density mixtures within nonparametric estimation frameworks \citep{wang2024nonparametric}.
}



\section*{Disclosure statement}
The authors have no conflicts of interest to declare.

\section*{Data availability statement}
The data and code used to perform the computational experiments presented in Section~\ref{sec:4} are accessible in open-source repositories indicated in this article and are reproducible.

\section*{Acknowledgments}
This research benefited from the support the grants CNS2023-144260 and PID2022-13724OB-I00, funded by MICIU/AEI /10.13039/501100011033 and first also by European Union NextGenerationEU/PRTR

\bibliographystyle{apalike}

\begin{thebibliography}{}

\bibitem[Cuesta et~al., 2025]{cuesta2025}
Cuesta, M., D'Ambrosio, C., Durban, M., Guerrero, V., and Trindade, R.~S. (2025).
\newblock On leveraging constrained smooth additive regression models for global optimization.
\newblock arXiv:2510.14122v2.

\bibitem[De~Boor, 1978]{de1978practical}
De~Boor, C. (1978).
\newblock {\em A Practical Guide to Splines}, volume~27 of {\em Applied Mathematical Sciences}.
\newblock Springer New York.

\bibitem[Denison et~al., 1998]{denison1998automatic}
Denison, D., Mallick, B., and Smith, A. (1998).
\newblock Automatic {B}ayesian curve fitting.
\newblock {\em Journal of the Royal Statistical Society: Series B (Statistical Methodology)}, 60(2):333--350.

\bibitem[Dimitrova et~al., 2025]{GeDS}
Dimitrova, D.~S., Kaishev, V.~K., Lattuada, A., Guillén, E. L.~S., and Verrall, R.~J. (2025).
\newblock {\em GeDS: Geometrically Designed Spline Regression}.
\newblock R Core Team.
\newblock R package version 0.3.3.

\bibitem[Dimitrova et~al., 2023]{Dimitrova2023}
Dimitrova, D.~S., Kaishev, V.~K., Lattuada, A., and Verrall, R.~J. (2023).
\newblock Geometrically designed variable knot splines in generalized (non-)linear models.
\newblock {\em Applied Mathematics and Computation}, 436:127493.

\bibitem[Eilers and Marx, 1996]{eilers1996flexible}
Eilers, P.~H.~C. and Marx, B.~D. (1996).
\newblock Flexible smoothing with {B}-splines and penalties.
\newblock {\em Statistical Science}, 11(2):89--121.

\bibitem[Friedman, 1991]{friedman1991multivariate}
Friedman, J.~H. (1991).
\newblock Multivariate adaptive regression splines.
\newblock {\em The Annals of Statistics}, 19(1):1--67.

\bibitem[Frommlet and Nuel, 2016]{frommlet2016adaptive}
Frommlet, F. and Nuel, G. (2016).
\newblock An adaptive ridge procedure for {$L_0$} regularization.
\newblock {\em PloS ONE}, 11(2):e0148620.

\bibitem[Gaillard et~al., 2020]{R-opera}
Gaillard, P., Goude, Y., Plagne, L., Dubois, T., and Thieurmel, B. (2020).
\newblock {\em opera: Online Prediction by Expert Aggregation}.
\newblock R package version 1.2.0.

\bibitem[Giust and Jüttler, 2022]{GIUST2022114554}
Giust, A. and Jüttler, B. (2022).
\newblock Weighted isogeometric collocation based on spline projectors.
\newblock {\em Computer Methods in Applied Mechanics and Engineering}, 391:114554.

\bibitem[Goepp et~al., 2025]{goepp2025spline}
Goepp, V., Bouaziz, O., and Nuel, G. (2025).
\newblock Spline regression with automatic knot selection.
\newblock {\em Computational Statistics \& Data Analysis}, 202:108043.

\bibitem[Green, 1995]{green1995reversible}
Green, P.~J. (1995).
\newblock Reversible jump {M}arkov {C}hain {M}onte {C}arlo computation and {B}ayesian model determination.
\newblock {\em Biometrika}, 82(4):711--732.

\bibitem[Hastie and Tibshirani, 1986]{hastie1986generalized}
Hastie, T. and Tibshirani, R. (1986).
\newblock Generalized additive models.
\newblock {\em Statistical Science}, 1(3):297--310.

\bibitem[Hollaway and Killick, 2024]{hollaway2024detection}
Hollaway, M.~J. and Killick, R. (2024).
\newblock Detection of spatiotemporal changepoints: a generalised additive model approach.
\newblock {\em Statistics and Computing}, 34(5):162.

\bibitem[Jupp, 1978]{jupp1978approximation}
Jupp, D.~L. (1978).
\newblock Approximation to data by splines with free knots.
\newblock {\em SIAM Journal on Numerical Analysis}, 15(2):328--343.

\bibitem[Kim et~al., 2026]{kim2026simultaneous}
Kim, H., Bang, S., and Kang, J. (2026).
\newblock Simultaneous selection of knots and variables in additive models.
\newblock {\em Computational Statistics}, 41:10.

\bibitem[Leisch et~al., 2026]{mlbench}
Leisch, F., Dimitriadou, E., and Hornik, K. (2026).
\newblock {\em mlbench: Machine Learning Benchmark Problems}.
\newblock R Core Team.
\newblock R package version 2.1-8.

\bibitem[Magistris et~al., 2025]{magistris2025adaptive}
Magistris, A.~D., Romano, E., and Campagna, R. (2025).
\newblock Adaptive {G}eneralized {P}-{S}plines for {F}unctional {D}ata: {A} {S}tatistical {F}ramework via {B}lockwise {GSVD}.
\newblock {\em Statistics and Computing}, 35(6):203.

\bibitem[Navarro-Garc{\'\i}a et~al., 2024]{navarro2024mathematical}
Navarro-Garc{\'\i}a, M., Guerrero, V., and Durban, M. (2024).
\newblock A mathematical optimization approach to shape-constrained generalized additive models.
\newblock {\em Expert Systems with Applications}, 255:124654.

\bibitem[Navarro-García et~al., 2023]{NAVARROGARCIA}
Navarro-García, M., Guerrero, V., and Durban, M. (2023).
\newblock On constrained smoothing and out-of-range prediction using {P}-splines: A conic optimization approach.
\newblock {\em Applied Mathematics and Computation}, 441:127679.

\bibitem[Nelder and Wedderburn, 1972]{nelder1972generalized}
Nelder, J.~A. and Wedderburn, R.~W. (1972).
\newblock Generalized linear models.
\newblock {\em Journal of the Royal Statistical Society Series A: Statistics in Society}, 135(3):370--384.

\bibitem[Osborne et~al., 1998]{osborne1998knot}
Osborne, M.~R., Presnell, B., and Turlach, B.~A. (1998).
\newblock Knot selection for regression splines via the lasso.
\newblock {\em Computing Science and Statistics}, pages 44--49.

\bibitem[Rippe et~al., 2012]{rippe2012visualization}
Rippe, R.~C., Meulman, J.~J., and Eilers, P.~H.~C. (2012).
\newblock Visualization of genomic changes by segmented smoothing using an {$L_0$} penalty.
\newblock {\em PloS ONE}, 7(6):e38230.

\bibitem[Rodr{\'\i}guez-{\'A}lvarez et~al., 2019]{rodriguez2019estimation}
Rodr{\'\i}guez-{\'A}lvarez, M.~X., Durban, M., Lee, D.-J., and Eilers, P.~H. (2019).
\newblock On the estimation of variance parameters in non-standard generalised linear mixed models: Application to penalised smoothing.
\newblock {\em Statistics and Computing}, 29:483--500.

\bibitem[Seck and Denuit, 2022]{seck2022adaptive}
Seck, N.~A. and Denuit, M. (2022).
\newblock Adaptive {S}plines for {C}ontinuous {F}eatures in {R}isk {A}ssessment.
\newblock \textit{{CAS} {E}-{F}orum}, Summer. Available at: \url{https://forum.casact.org/article/38389-adaptive-splines-for-continuous-features-in-risk-assessment}.

\bibitem[Spiriti et~al., 2013]{spiriti2013knot}
Spiriti, S., Eubank, R., Smith, P.~W., and Young, D. (2013).
\newblock Knot selection for least-squares and penalized splines.
\newblock {\em Journal of Statistical Computation and Simulation}, 83(6):1020--1036.

\bibitem[Thielmann et~al., 2025]{thielmann2025enhancing}
Thielmann, A., Kneib, T., and S{\"a}fken, B. (2025).
\newblock Enhancing adaptive spline regression: an evolutionary approach to optimal knot placement and smoothing parameter selection.
\newblock {\em Journal of Computational and Graphical Statistics}, pages 1--16.

\bibitem[Tibshirani, 1996]{tibshirani1996regression}
Tibshirani, R. (1996).
\newblock Regression shrinkage and selection via the lasso.
\newblock {\em Journal of the Royal Statistical Society Series B: Statistical Methodology}, 58(1):267--288.

\bibitem[Wang et~al., 2024]{wang2024nonparametric}
Wang, X., Zhao, Y., Ni, Q., and Tang, S. (2024).
\newblock Nonparametric density estimation with nonuniform b-spline bases.
\newblock {\em Journal of Computational and Applied Mathematics}, 440:115648.

\bibitem[Wood, 2017]{wood2017generalizedbook}
Wood, S.~N. (2017).
\newblock {\em Generalized Additive Models: An Introduction with {R}}.
\newblock Chapman and Hall/CRC, Boca Raton, FL, 2 edition.

\bibitem[Wood, 2020]{wood2020inference}
Wood, S.~N. (2020).
\newblock Inference and computation with generalized additive models and their extensions.
\newblock {\em Test}, 29(2):307--339.

\bibitem[Wood, 2025]{mgcvR}
Wood, S.~N. (2025).
\newblock {\em mgcv: Mixed GAM Computation Vehicle with Automatic Smoothness Estimation}.
\newblock R Core Team.
\newblock R package version 1.9-3.

\bibitem[Wood and Fasiolo, 2017]{wood2017generalized}
Wood, S.~N. and Fasiolo, M. (2017).
\newblock A generalized {F}ellner-{S}chall method for smoothing parameter optimization with application to tweedie location, scale and shape models.
\newblock {\em Biometrics}, 73(4):1071--1081.

\bibitem[Yuan et~al., 2013]{yuan2013adaptive}
Yuan, Y., Chen, N., and Zhou, S. (2013).
\newblock Adaptive {B}-spline knot selection using multi-resolution basis set.
\newblock {\em IIE Transactions}, 45(12):1263--1277.

\end{thebibliography}

\newpage

\begin{appendices}
\section{Additional results for Section \ref{subsec: 4.1}} \label{app: Sim Setting}

This appendix presents additional results complementing the analyses of Simulations 1 and 2.

\subsection{{ Simulation} 1: Gaussian}


{
Figures~\ref{fig: Boxplot_Sim1_n300_SNR1}, \ref{fig: Boxplot_Sim1_n300_SNR2}, \ref{fig: Boxplot_Sim1_n300_SNR4}, \ref{fig: Boxplot_Sim1_n500_SNR1}, \ref{fig: Boxplot_Sim1_n500_SNR2},\ref{fig: Boxplot_Sim1_n500_SNR4}, \ref{fig: Boxplot_Sim1_n800_SNR1}, \ref{fig: Boxplot_Sim1_n800_SNR2} and \ref{fig: Boxplot_Sim1_n800_SNR4} present the boxplots of the performance metrics over the 100 replicates, together with the mean estimated smooth function for each of the six covariates (Figures~\ref{fig: Covariates_Sim1_n300_SNR1}, \ref{fig: Covariates_Sim1_n300_SNR2}, \ref{fig: Covariates_Sim1_n300_SNR4}, \ref{fig: Covariates_Sim1_n500_SNR1}, \ref{fig: Covariates_Sim1_n500_SNR2},\ref{fig: Covariates_Sim1_n500_SNR4}, \ref{fig: Covariates_Sim1_n800_SNR1}, \ref{fig: Covariates_Sim1_n800_SNR2} and \ref{fig: Covariates_Sim1_n800_SNR4}), for all combinations of $n\in\{300,500,800\}$ and $\mathrm{SNR}\in\{1,2,4\}$.
}

\subsubsection{$n=300,\;\text{SNR}=1$}
\begin{figure}[H]
    \centering
    \includegraphics[width=0.8\textwidth]{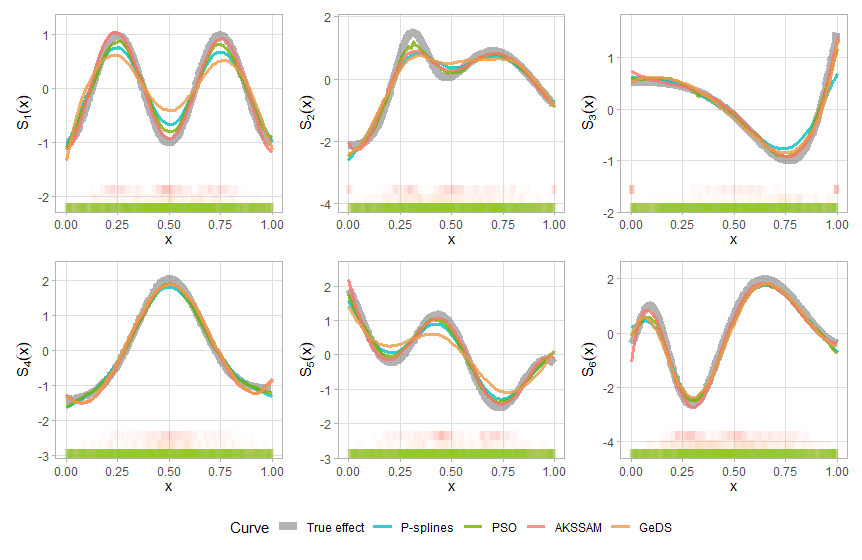}
    \caption{Mean estimated smooth functions for each covariate and distribution of the selected knots over 100 replicates in Simulation 1, $n=300,\;\text{SNR}=1$.}
    \label{fig: Covariates_Sim1_n300_SNR1}
\end{figure}
\begin{figure}[H]
    \centering
    \includegraphics[width=0.8\textwidth]{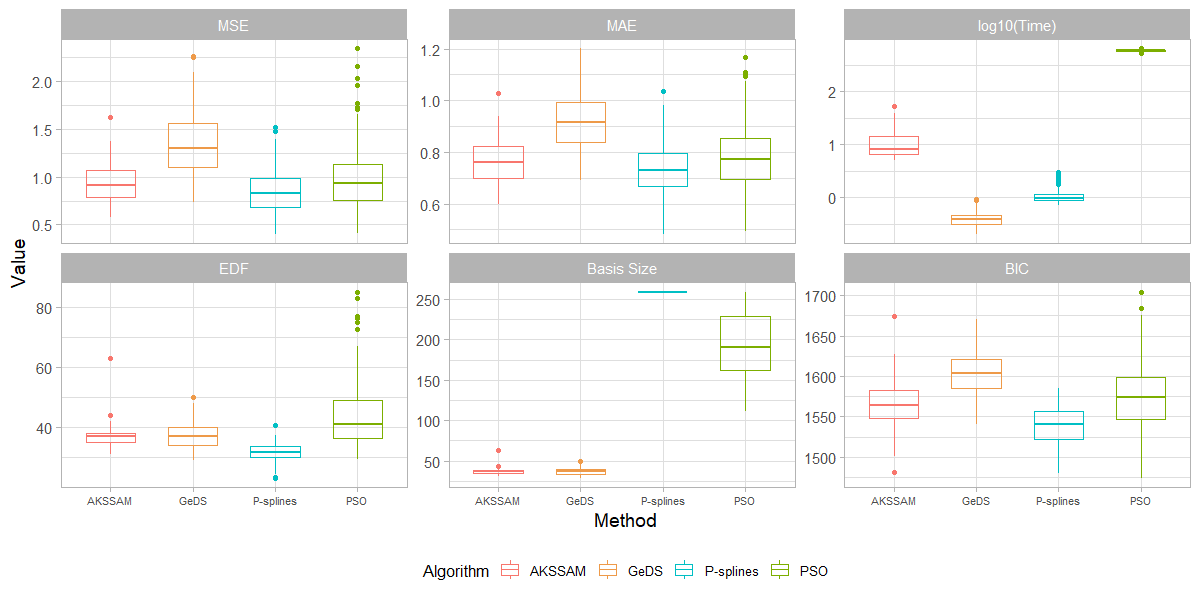}
    \caption{Boxplots of the performance metrics over 100 replicates for  Simulation 1, $n=300,\;\text{SNR}=1$.}
    \label{fig: Boxplot_Sim1_n300_SNR1}
\end{figure}

\subsubsection{$n=300,\;\text{SNR}=2$}
\begin{figure}[H]
    \centering
    \includegraphics[width=0.8\textwidth]{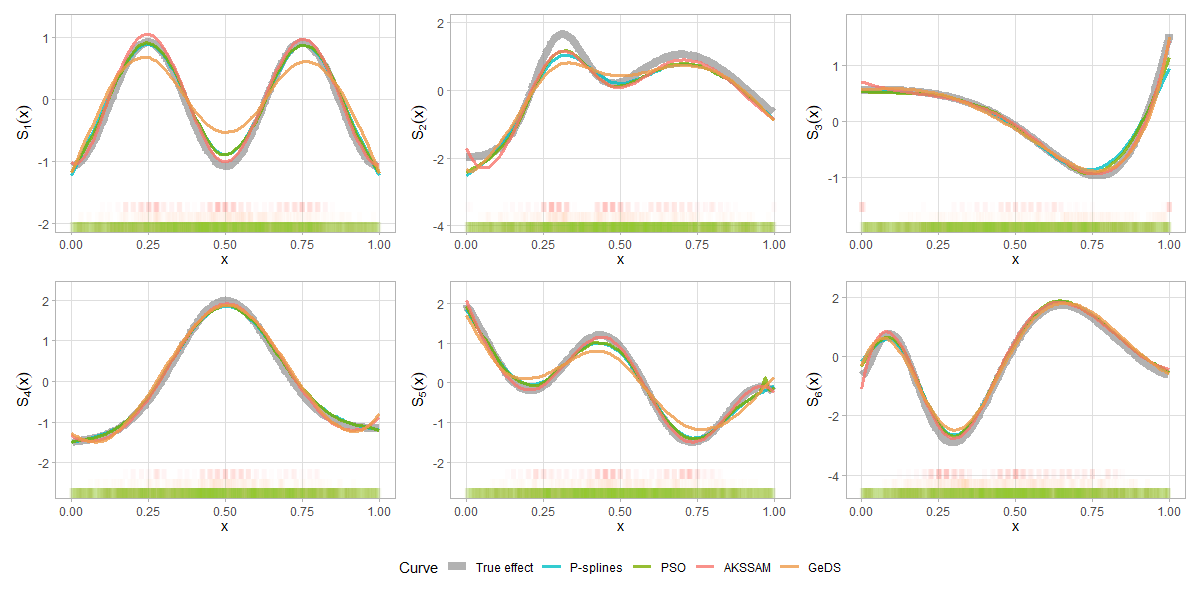}
    \caption{Mean estimated smooth functions for each covariate and distribution of the selected knots over 100 replicates in Simulation 1, $n=300,\;\text{SNR}=2$.}
    \label{fig: Covariates_Sim1_n300_SNR2}
\end{figure}
\begin{figure}[H]
    \centering
    \includegraphics[width=0.8\textwidth]{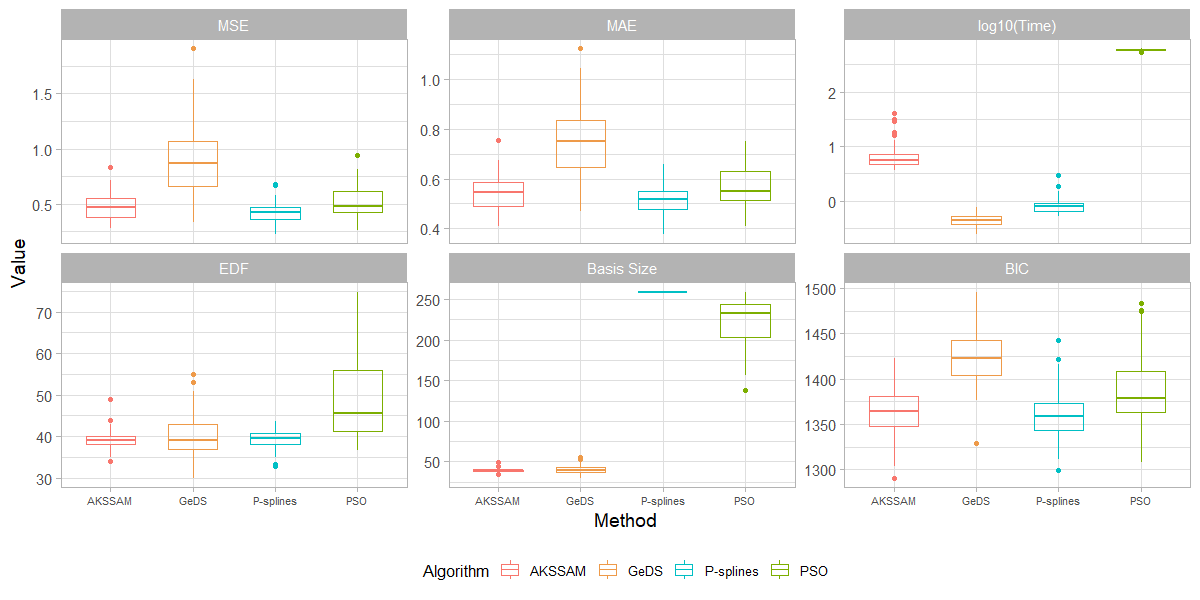}
    \caption{Boxplots of the performance metrics over 100 replicates for  Simulation 1, $n=300,\;\text{SNR}=2$.}
    \label{fig: Boxplot_Sim1_n300_SNR2}
\end{figure}

\subsubsection{$n=300,\;\text{SNR}=3$}
\begin{figure}[H]
    \centering
    \includegraphics[width=0.8\textwidth]{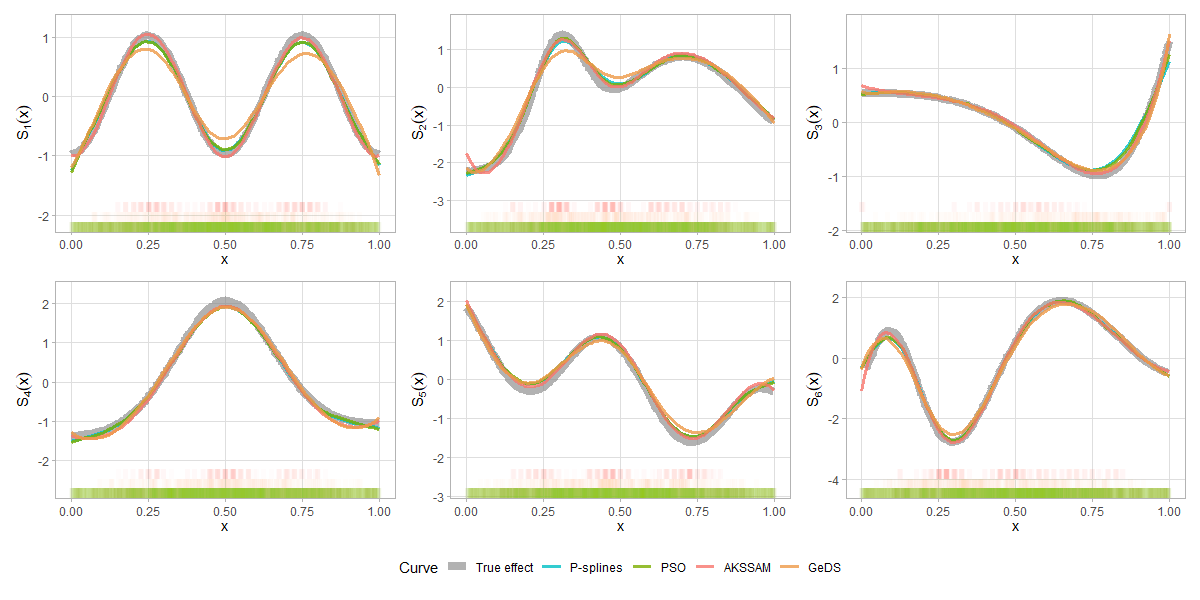}
    \caption{Mean estimated smooth functions for each covariate and distribution of the selected knots over 100 replicates in Simulation 1, $n=300,\;\text{SNR}=4$.}
    \label{fig: Covariates_Sim1_n300_SNR4}
\end{figure}
\begin{figure}[H]
    \centering
    \includegraphics[width=0.8\textwidth]{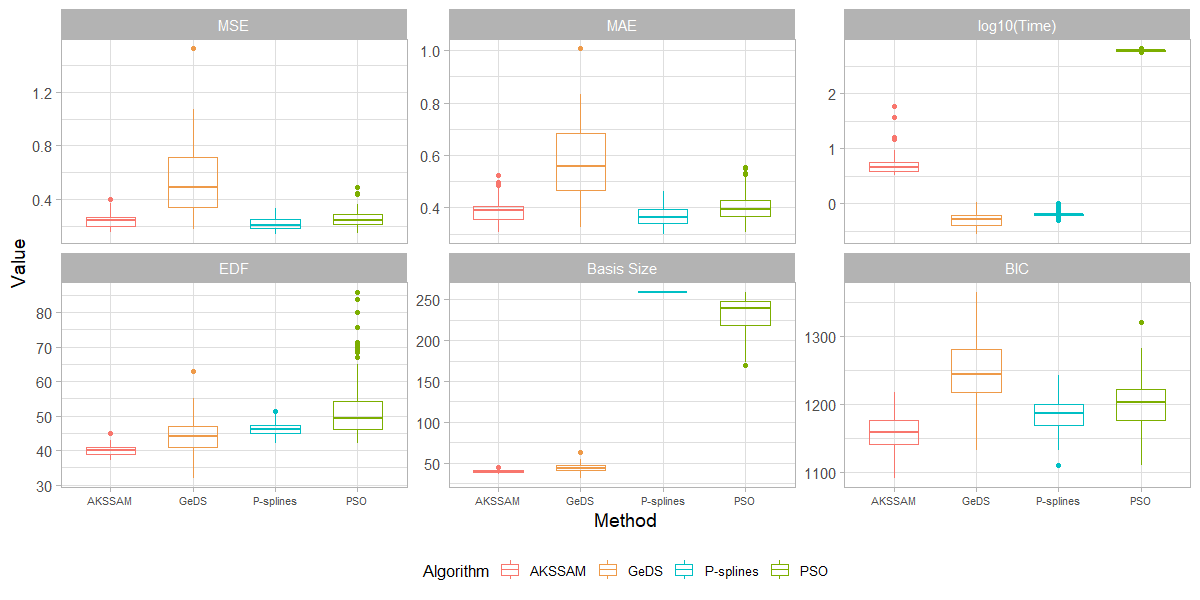}
    \caption{Boxplots of the performance metrics over 100 replicates for  Simulation 1, $n=300,\;\text{SNR}=4$.}
    \label{fig: Boxplot_Sim1_n300_SNR4}
\end{figure}

\subsubsection{$n=500,\;\text{SNR}=1$}
\begin{figure}[H]
    \centering
    \includegraphics[width=0.8\textwidth]{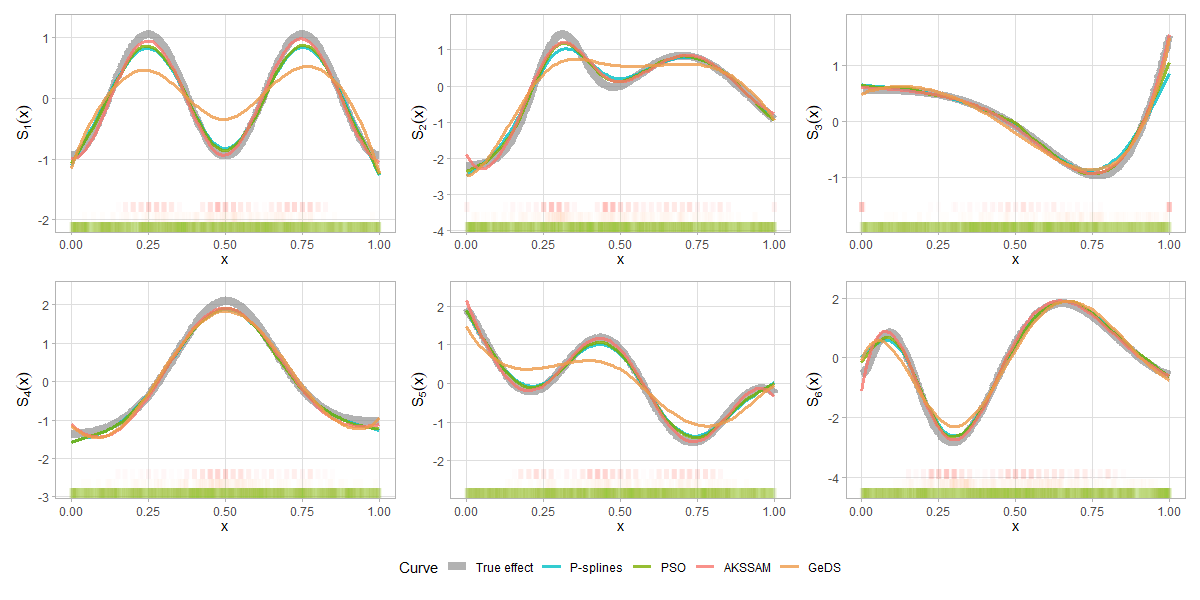}
    \caption{Mean estimated smooth functions for each covariate and distribution of the selected knots over 100 replicates in Simulation 1, $n=500,\;\text{SNR}=1$.}
    \label{fig: Covariates_Sim1_n500_SNR1}
\end{figure}
\begin{figure}[H]
    \centering
    \includegraphics[width=0.8\textwidth]{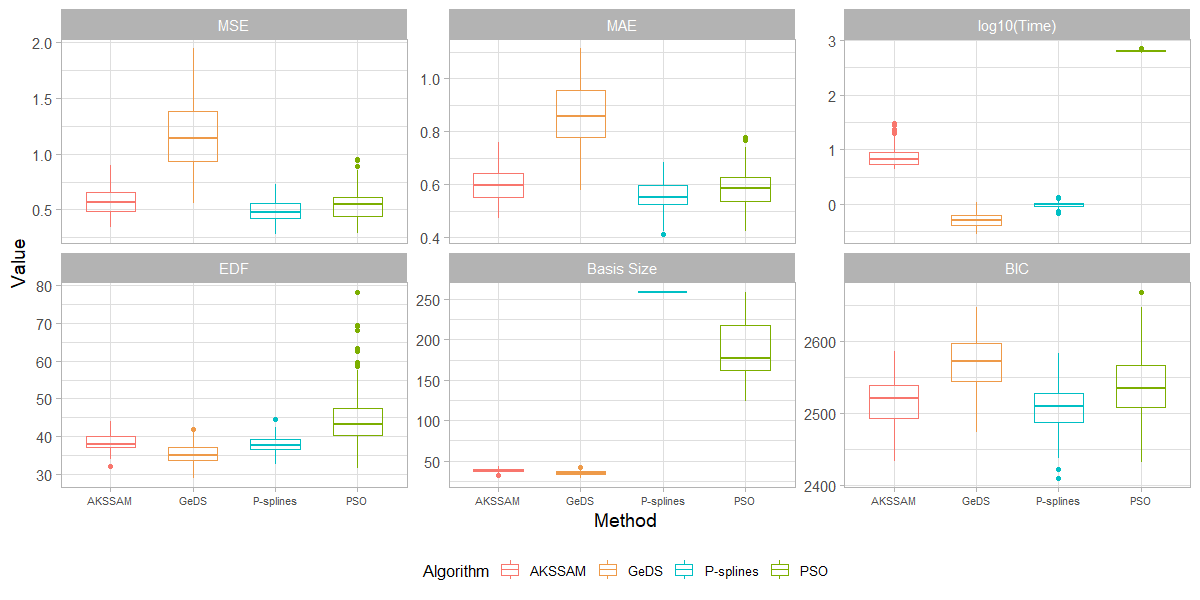}
    \caption{Boxplots of the performance metrics over 100 replicates for  Simulation 1, $n=500,\;\text{SNR}=1$.}
    \label{fig: Boxplot_Sim1_n500_SNR1}
\end{figure}

\subsubsection{$n=500,\;\text{SNR}=2$}
\begin{figure}[H]
    \centering
    \includegraphics[width=0.8\textwidth]{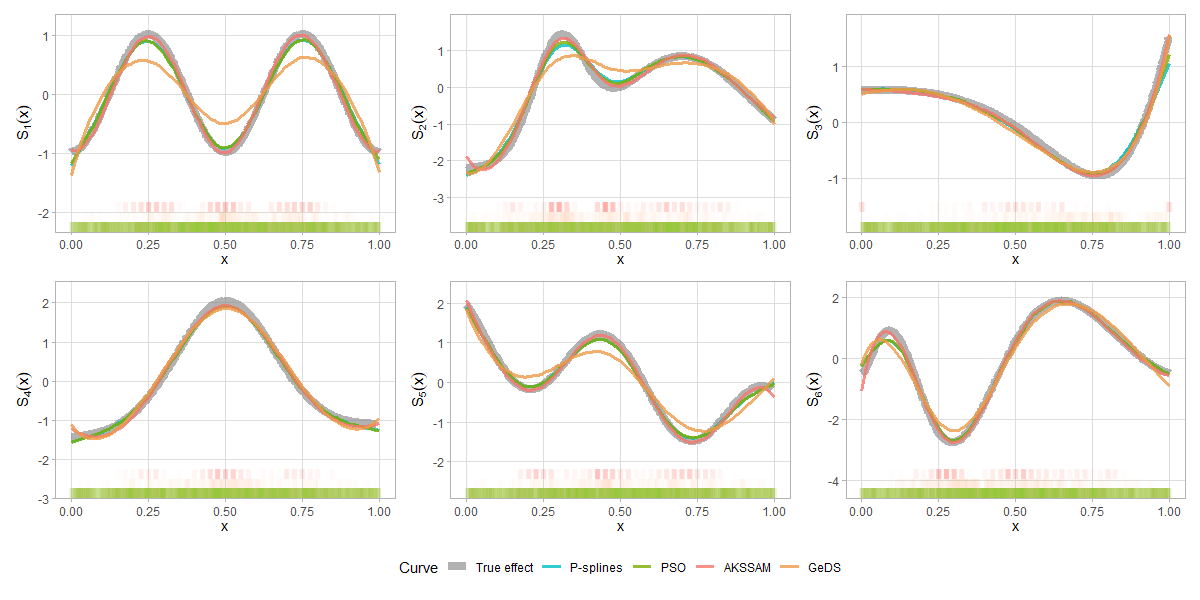}
    \caption{Mean estimated smooth functions for each covariate and distribution of the selected knots over 100 replicates in Simulation 1, $n=500,\;\text{SNR}=2$.}
    \label{fig: Covariates_Sim1_n500_SNR2}
\end{figure}
\begin{figure}[H]
    \centering
    \includegraphics[width=0.8\textwidth]{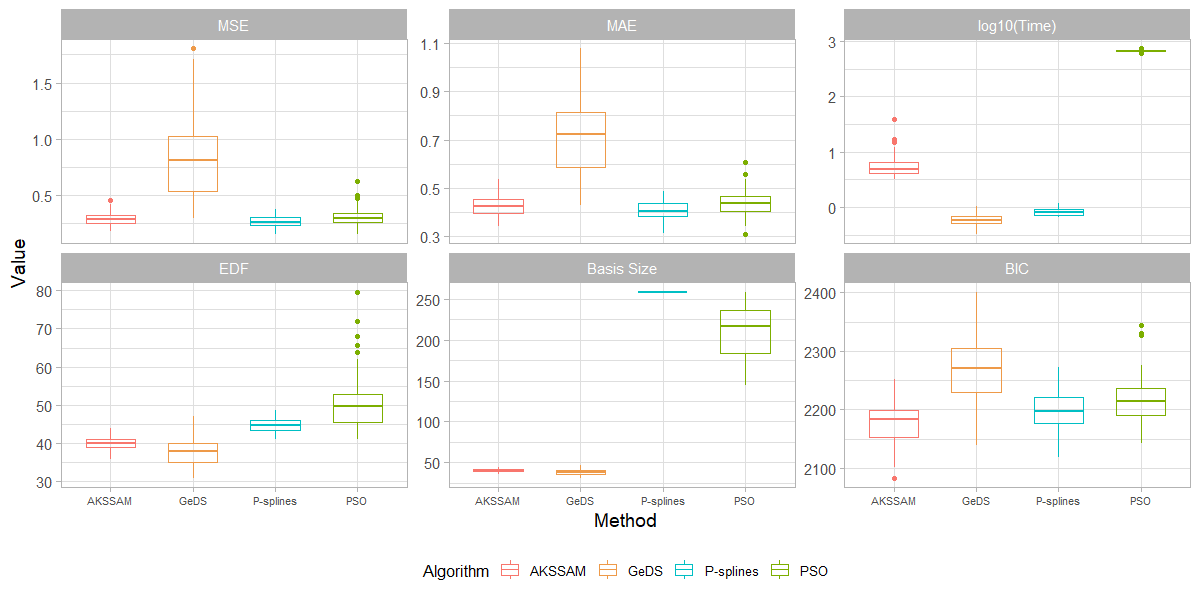}
    \caption{Boxplots of the performance metrics over 100 replicates for  Simulation 1, $n=500,\;\text{SNR}=2$.}
    \label{fig: Boxplot_Sim1_n500_SNR2}
\end{figure}

\subsubsection{$n=500,\;\text{SNR}=4$}
\begin{figure}[H]
    \centering
    \includegraphics[width=0.8\textwidth]{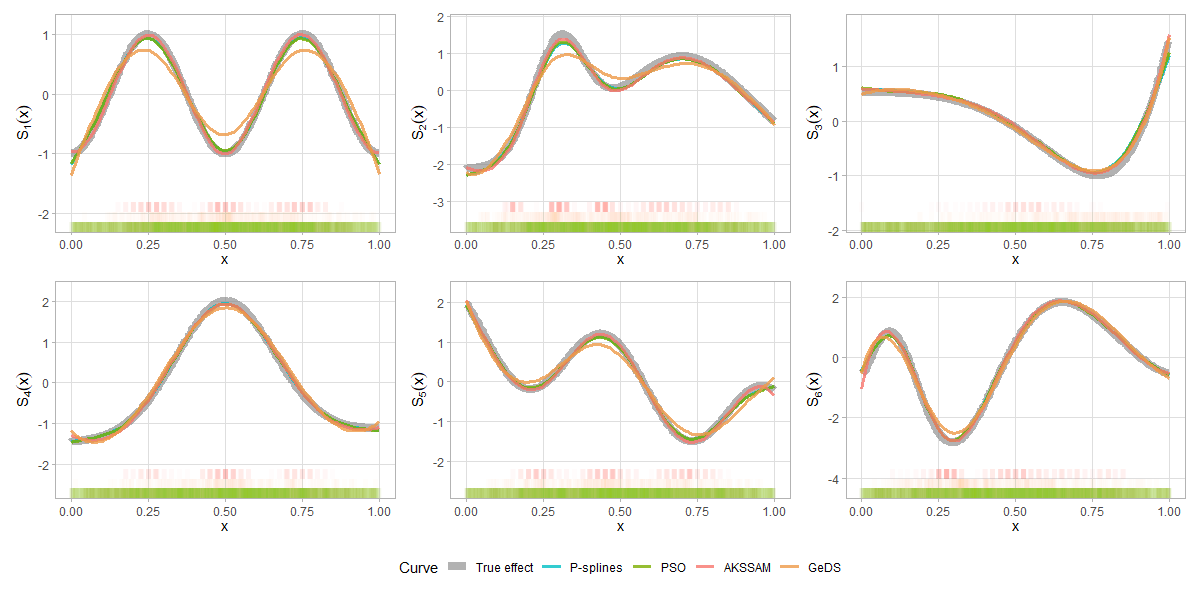}
    \caption{Mean estimated smooth functions for each covariate and distribution of the selected knots over 100 replicates in Simulation 1, $n=500,\;\text{SNR}=4$.}
    \label{fig: Covariates_Sim1_n500_SNR4}
\end{figure}
\begin{figure}[H]
    \centering
    \includegraphics[width=0.8\textwidth]{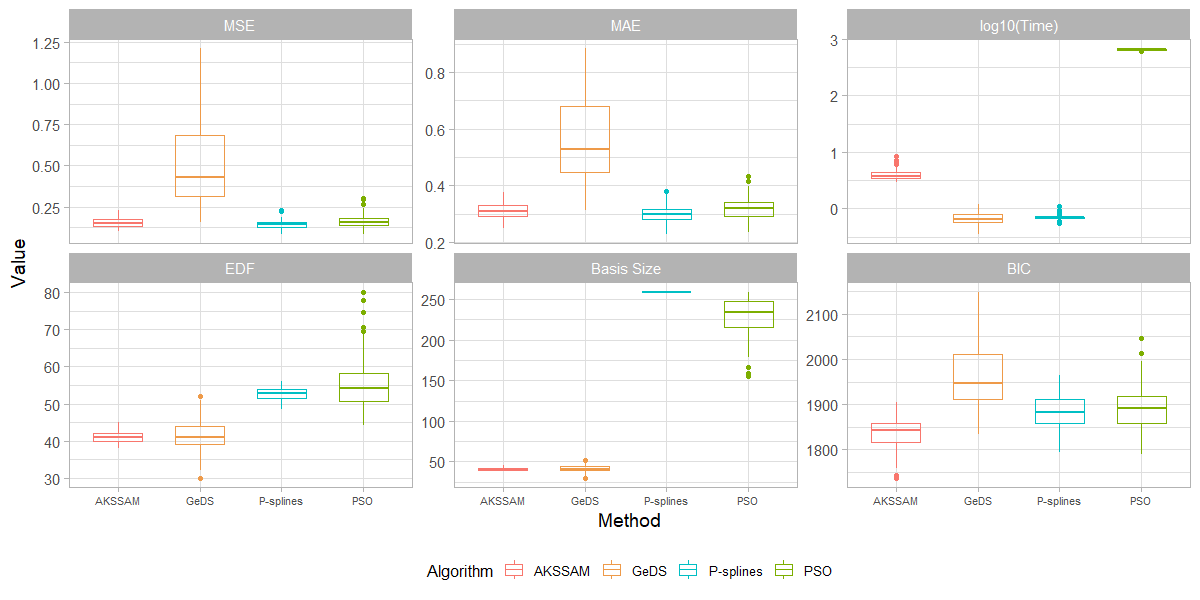}
    \caption{Boxplots of the performance metrics over 100 replicates for  Simulation 1, $n=500,\;\text{SNR}=4$.}
    \label{fig: Boxplot_Sim1_n500_SNR4}
\end{figure}

\subsubsection{$n=800,\;\text{SNR}=1$}
\begin{figure}[H]
    \centering
    \includegraphics[width=0.8\textwidth]{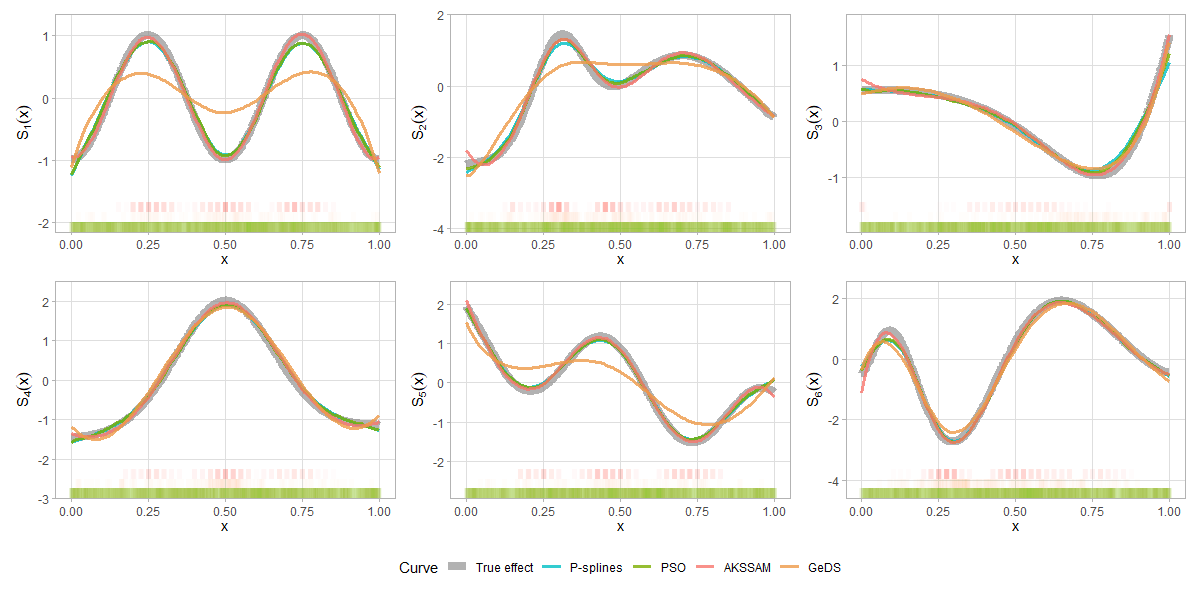}
    \caption{Mean estimated smooth functions for each covariate and distribution of the selected knots over 100 replicates in Simulation 1, $n=800,\;\text{SNR}=1$.}
    \label{fig: Covariates_Sim1_n800_SNR1}
\end{figure}
\begin{figure}[H]
    \centering
    \includegraphics[width=0.8\textwidth]{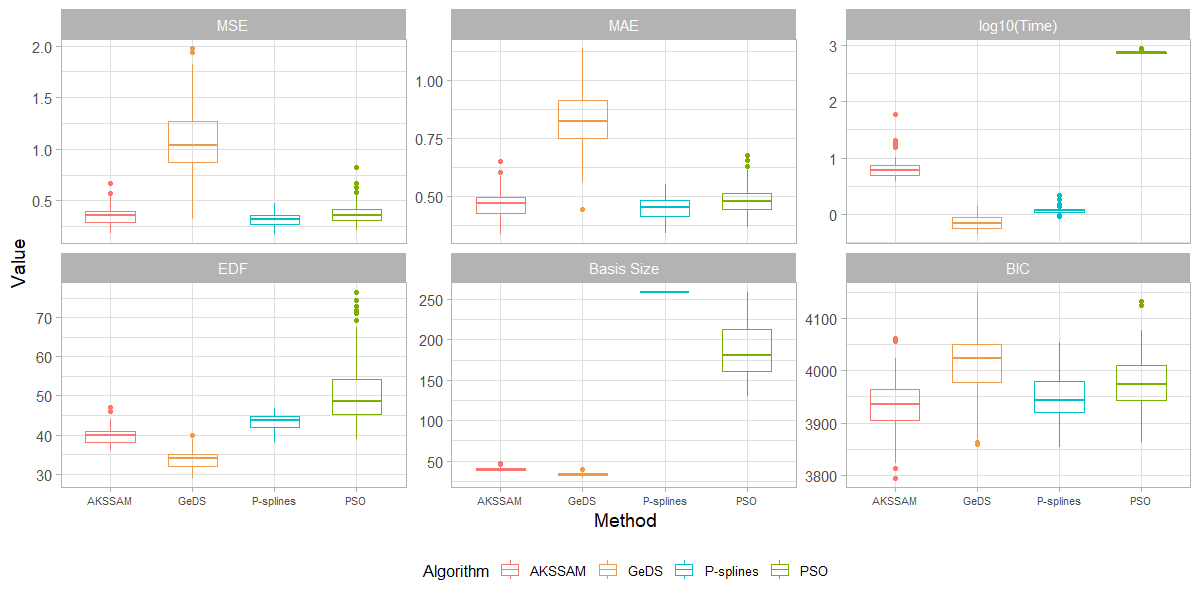}
    \caption{Boxplots of the performance metrics over 100 replicates for  Simulation 1, $n=800,\;\text{SNR}=1$.}
    \label{fig: Boxplot_Sim1_n800_SNR1}
\end{figure}

\subsubsection{$n=800,\;\text{SNR}=2$}
\begin{figure}[H]
    \centering
    \includegraphics[width=0.8\textwidth]{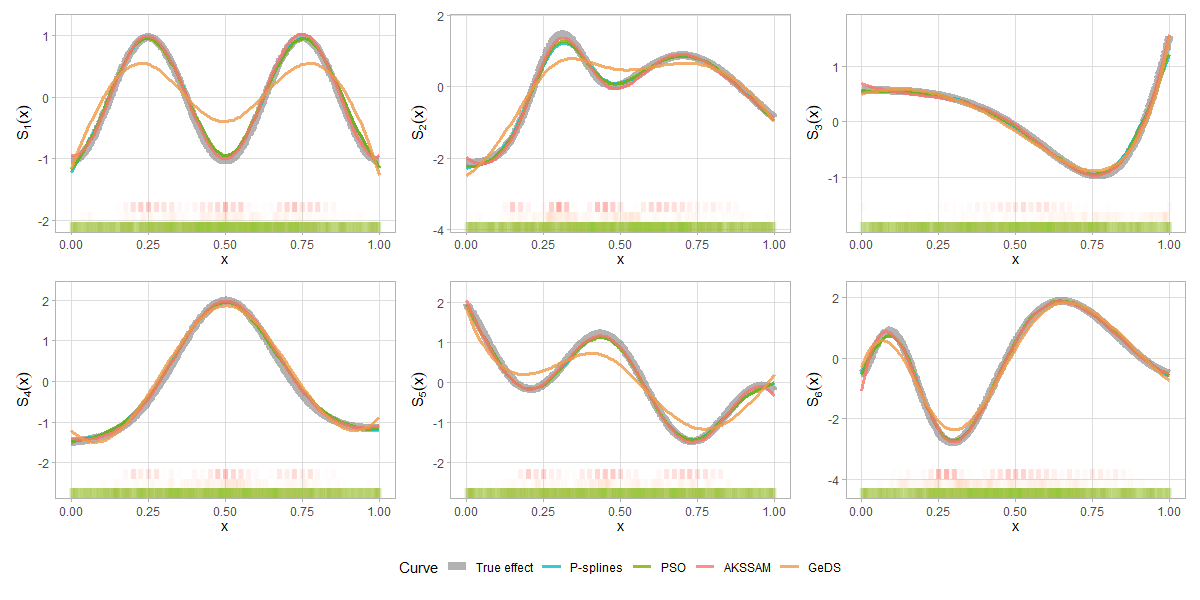}
    \caption{Mean estimated smooth functions for each covariate and distribution of the selected knots over 100 replicates in Simulation 1, $n=800,\;\text{SNR}=2$.}
    \label{fig: Covariates_Sim1_n800_SNR2}
\end{figure}
\begin{figure}[H]
    \centering
    \includegraphics[width=0.8\textwidth]{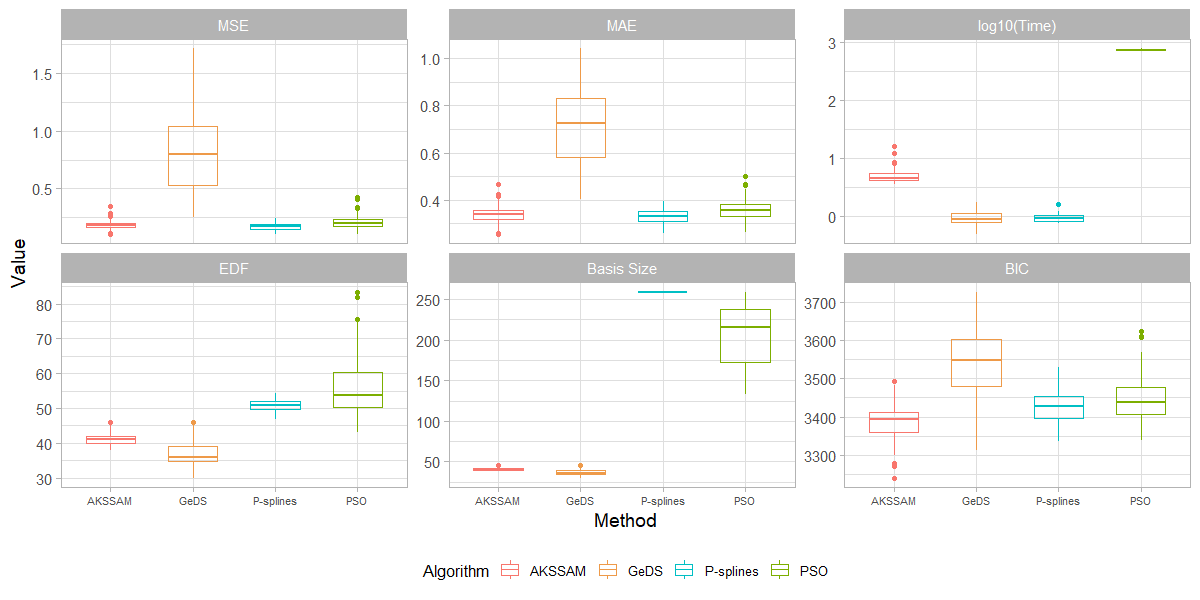}
    \caption{Boxplots of the performance metrics over 100 replicates for  Simulation 1, $n=800,\;\text{SNR}=2$.}
    \label{fig: Boxplot_Sim1_n800_SNR2}
\end{figure}

\subsubsection{$n=800,\;\text{SNR}=4$}
\begin{figure}[H]
    \centering
    \includegraphics[width=0.8\textwidth]{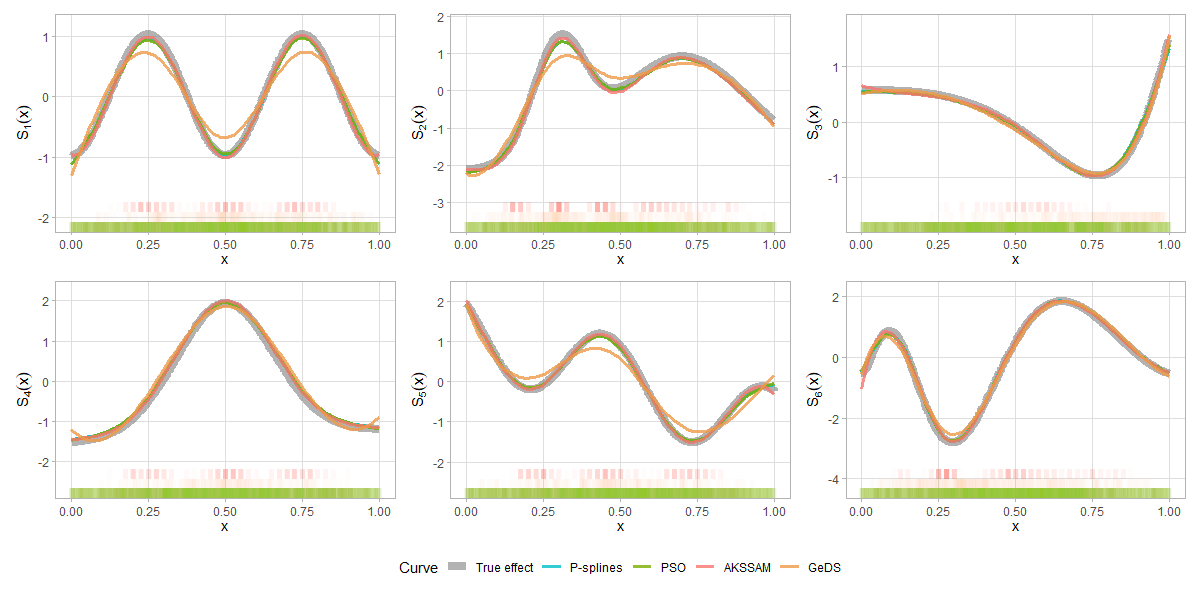}
    \caption{Mean estimated smooth functions for each covariate and distribution of the selected knots over 100 replicates in Simulation 1, $n=800,\;\text{SNR}=4$.}
    \label{fig: Covariates_Sim1_n800_SNR4}
\end{figure}
\begin{figure}[H]
    \centering
    \includegraphics[width=0.8\textwidth]{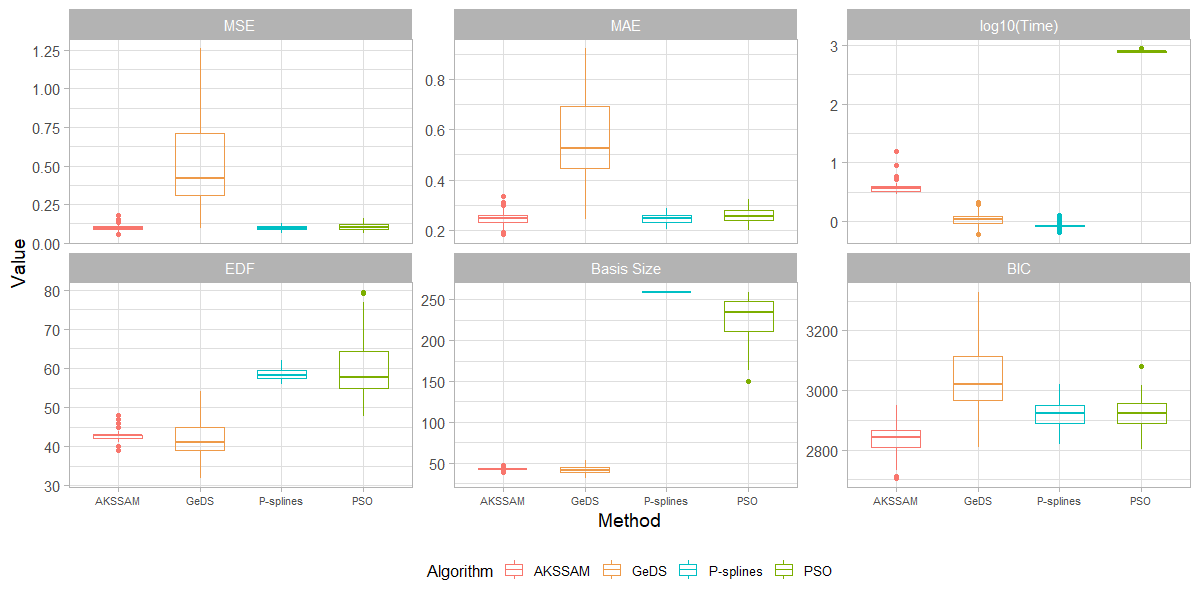}
    \caption{Boxplots of the performance metrics over 100 replicates for  Simulation 1, $n=800,\;\text{SNR}=4$.}
    \label{fig: Boxplot_Sim1_n800_SNR4}
\end{figure}

\subsection{{Simulation} 2: Poisson}

{Figures~\ref{fig: Boxplot_Sim2_n300}, \ref{fig: Boxplot_Sim2_n500} and \ref{fig: Boxplot_Sim2_n800} present the boxplots of the performance metrics over the 100 replicates, togehter with the mean estimated smooth fucntions for each of the three covariates (Figures~\ref{fig: Covariates_Sim2_n300}, \ref{fig: Covariates_Sim2_n500} and \ref{fig: Covariates_Sim2_n800}), for the different sample sizes $n\in\{300,500,800\}$.}


\subsubsection{$n=300$}

\begin{figure}[H]
    \centering
    \includegraphics[width=0.8\textwidth]{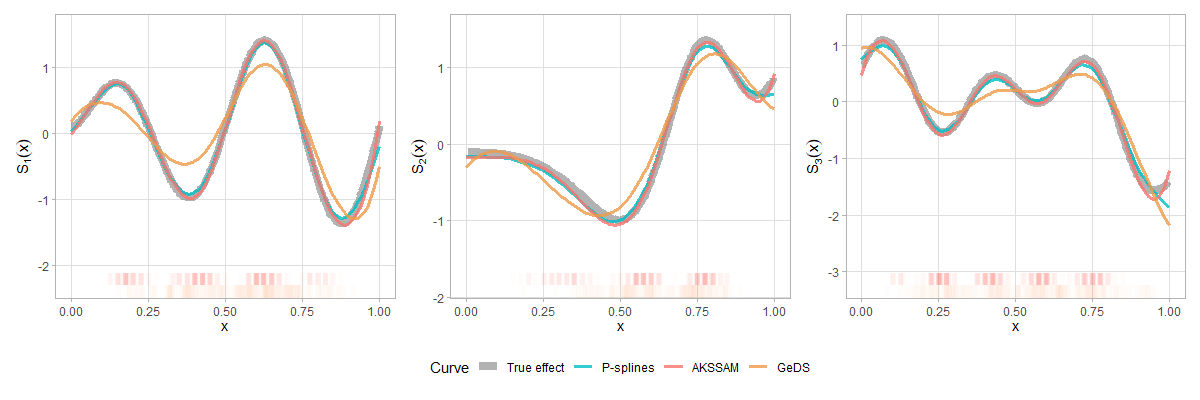}
    \caption{Mean estimated smooth functions for each covariate and distribution of the selected knots over 100 replicates in Simulation 2, $n=300$.}
    \label{fig: Covariates_Sim2_n300}
\end{figure}
\begin{figure}[H]
    \centering
    \includegraphics[width=0.8\textwidth]{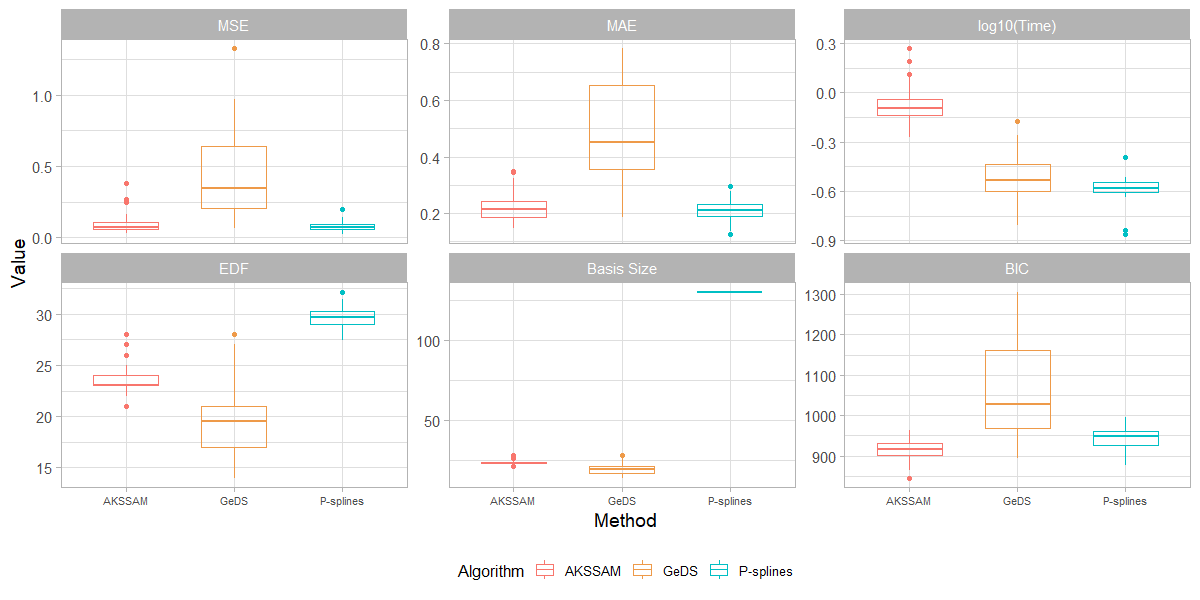}
    \caption{Boxplots of the performance metrics over 100 replicates for  Simulation 2, $n=300$.}
    \label{fig: Boxplot_Sim2_n300}
\end{figure}

\subsubsection{$n=500$}

\begin{figure}[H]
    \centering
    \includegraphics[width=0.8\textwidth]{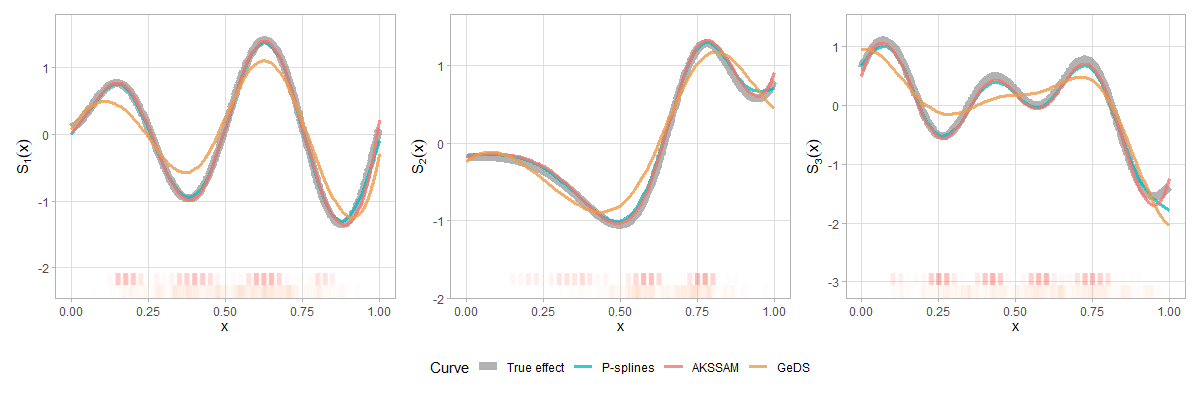}
    \caption{Mean estimated smooth functions for each covariate and distribution of the selected knots over 100 replicates in Simulation 2, $n=500$.}
    \label{fig: Covariates_Sim2_n500}
\end{figure}
\begin{figure}[H]
    \centering
    \includegraphics[width=0.8\textwidth]{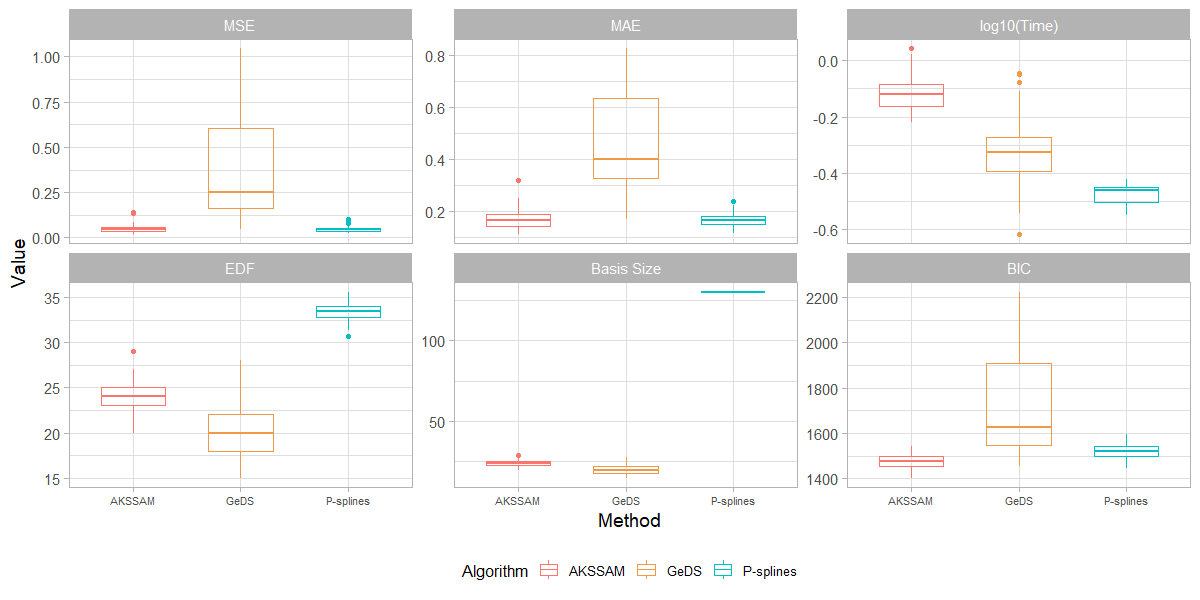}
    \caption{Boxplots of the performance metrics over 100 replicates for  Simulation 2, $n=500$.}
    \label{fig: Boxplot_Sim2_n500}
\end{figure}

\subsubsection{$n=800$}

\begin{figure}[H]
    \centering
    \includegraphics[width=0.8\textwidth]{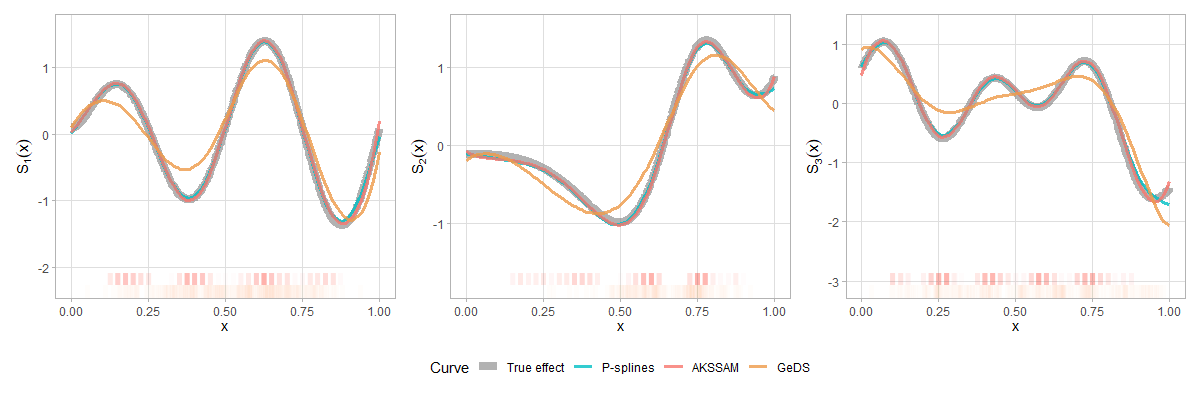}
    \caption{Mean estimated smooth functions for each covariate and distribution of the selected knots over 100 replicates in Simulation 2, $n=800$.}
    \label{fig: Covariates_Sim2_n800}
\end{figure}
\begin{figure}[H]
    \centering
    \includegraphics[width=0.8\textwidth]{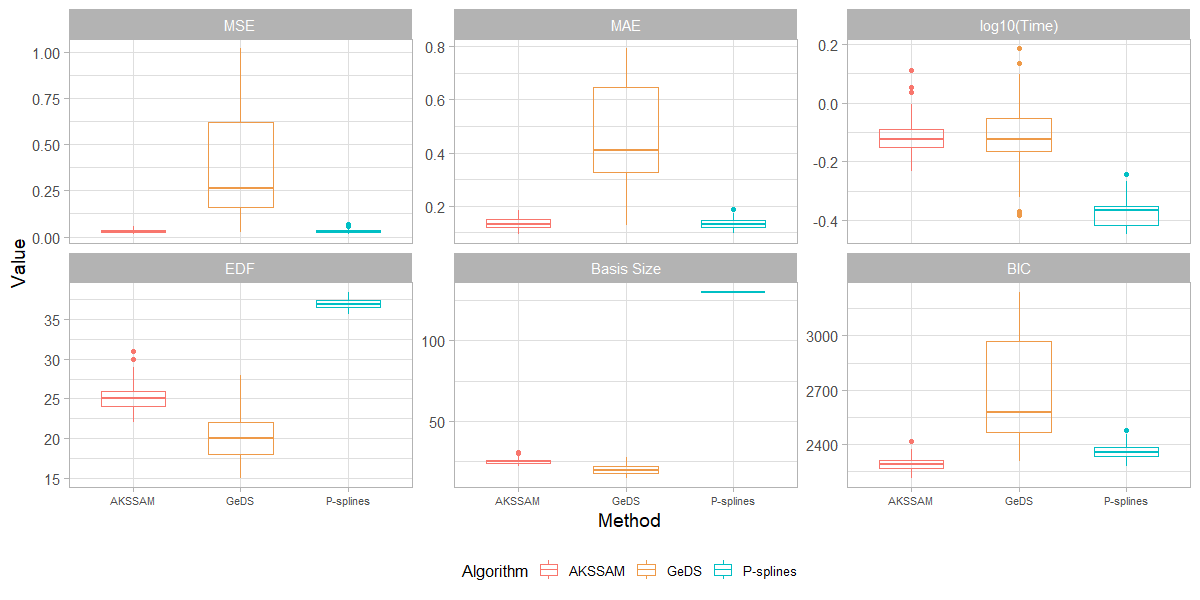}
    \caption{Boxplots of the performance metrics over 100 replicates for  Simulation 2, $n=800$.}
    \label{fig: Boxplot_Sim2_n800}
\end{figure}

\section{Additional results for Section \ref{subsec: 4.2}} \label{app: Real Setting}

This appendix presents additional results complementing the analyses of the real world datasets in Section \ref{subsec: 4.2}, namely \texttt{electric\_load} and \texttt{PimaIndians} datasets.

\subsection{\texttt{electric\_load} dataset}

Figure~\ref{fig: Boxplot_ElectricLoad} depicts the boxplots of the 
performance metrics across the five folds.

\begin{figure}[H]
    \centering
    \includegraphics[width=0.8\textwidth]{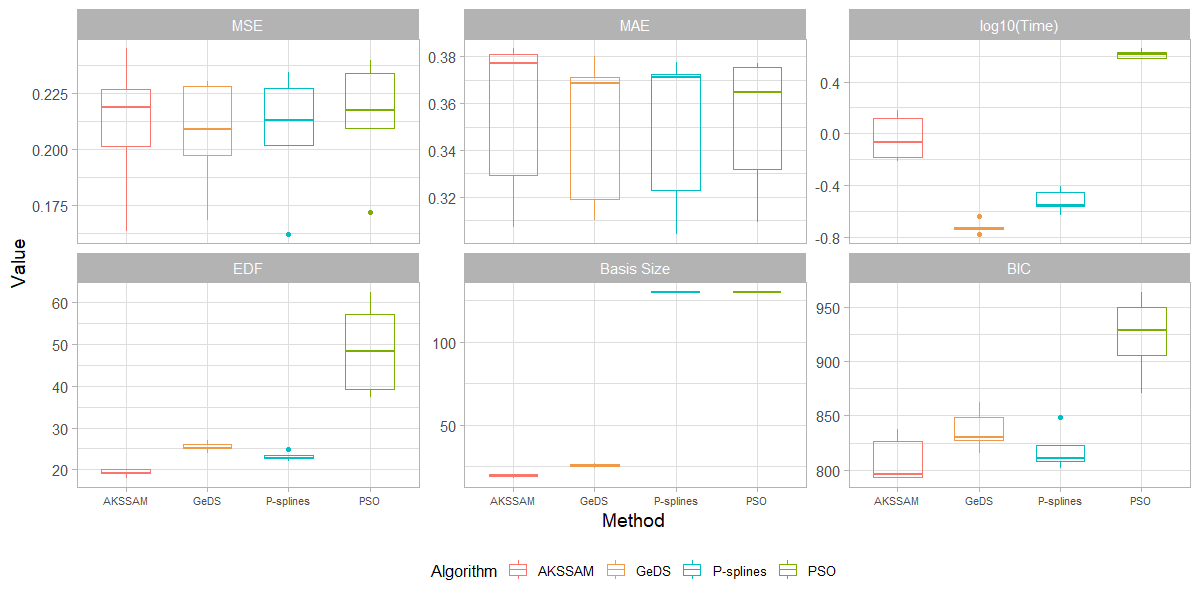}
    \caption{Boxplot of the performance metrics across the five folds in \texttt{electric\_load}.}
    \label{fig: Boxplot_ElectricLoad}
\end{figure}

\subsection{\texttt{PimaIndians} dataset}

Similarly to the \texttt{electric\_road} dataset, Figure~\ref{fig: Boxplot_PimaIndians} depict the boxplots of the performance metrics across the five folds


\begin{figure}[H]
    \centering
    \includegraphics[width=0.8\textwidth]{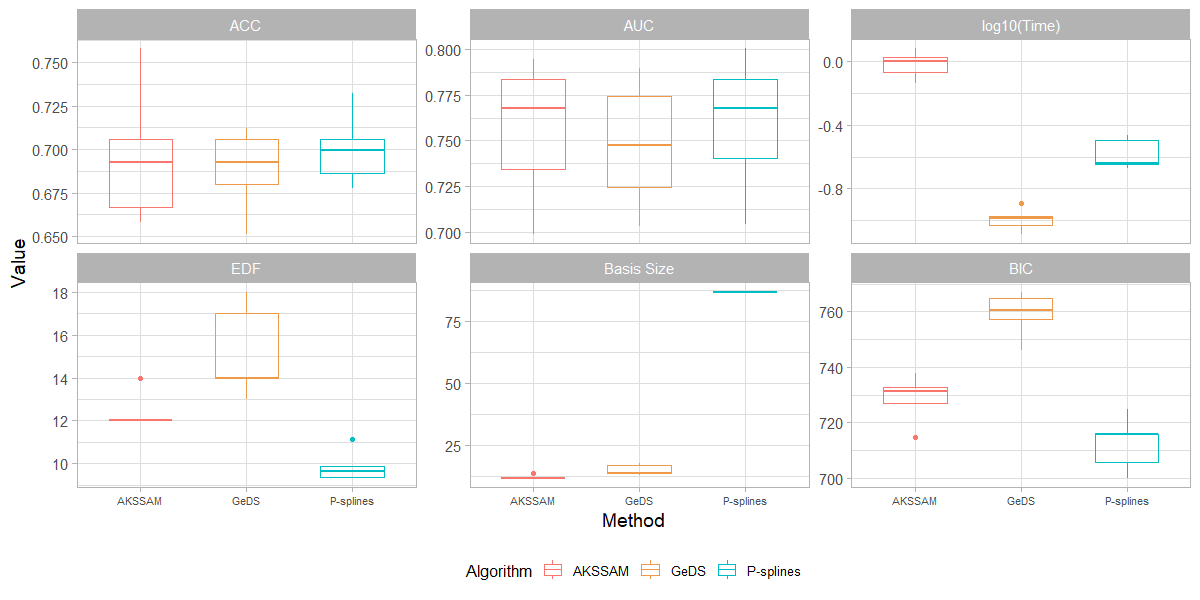}
    \caption{Boxplot of the performance metrics across the five folds in \texttt{PimaIndians}.}
    \label{fig: Boxplot_PimaIndians}
\end{figure}

\end{appendices}

\end{document}